\documentclass[journal]{IEEEtran}

\ifCLASSINFOpdf
\else
\fi

\usepackage[svgnames,table,dvipsnames,usenames]{xcolor} 
\definecolor{DarkGreen}{rgb}{0,0.5,0}
\definecolor{DarkRed}{rgb}{0.75,0,0}
\definecolor{DarkOrange}{RGB}{255, 128, 0}
\definecolor{LightBlue}{RGB}{0, 190, 247}
\usepackage{amsmath} %
\usepackage{nameref}
\usepackage{amssymb}
\usepackage{graphicx}
\usepackage{tikz}
\usepackage{placeins}
\usetikzlibrary{shapes,snakes}
\usepackage{url}
\usepackage[ruled,vlined,linesnumbered]{algorithm2e}
\usepackage{balance}
\usepackage[font=small,justification=justified,format=plain]{caption}
\usepackage{array}
\usepackage[shortlabels]{enumitem}

\newcommand{\add}[1]{#1}
\usepackage{amsthm}
\theoremstyle{definition}
\newtheorem{theorem}{Theorem}
\newtheorem{remark}{Remark}
\newtheorem{assumption}{Assumption}
\newcommand{\bee}{\begin{enumerate} \itemsep -1pt \topsep -2pt}
\newcommand{\eee}{\end{enumerate}}

\usepackage[flushmargin]{footmisc}

\usepackage[colorlinks]{hyperref}
\hypersetup{citecolor=Black}
\hypersetup{linkcolor=Black}
\hypersetup{urlcolor=DarkBlue}
\usepackage{caption}
\usepackage{mathtools}
\usepackage[capitalize]{cleveref}
\crefformat{equation}{(#2#1#3)}
\Crefformat{equation}{Equation~(#2#1#3)}
\Crefname{equation}{Equation}{Equations}
\usepackage{bm}
\newcommand{\bb}[1]{\boldsymbol{{#1}}}
\newcommand{\vect}[1]{\boldsymbol{{#1}}}

\newcommand{\dvect}[1]{\boldsymbol{{\dot{#1}}}}
\newcommand{\ddvect}[1]{\boldsymbol{{\ddot{#1}}}}
\newcommand{\dddvect}[1]{\boldsymbol{{\dddot{#1}}}}
\definecolor{MydarkRed}{RGB}{183,21,33}
\definecolor{MydarkGreen}{RGB}{59,143,50}
\definecolor{Myred}{RGB}{255,0,0}
\definecolor{MyorangeDarker}{RGB}{255,127,42}
\definecolor{MygreenDark}{RGB}{55,200,55}
\definecolor{MyorangeDark}{RGB}{255,212,42}
\definecolor{MyblueLight}{RGB}{85,221,255}
\definecolor{MylightGreen}{RGB}{161,255,191}
\definecolor{Myblue}{RGB}{0,0,255}
\newcommand{\tikzrectangle}[2][black,fill=red]{\tikz[baseline=0.0ex, line width=0.2mm]\draw[#1] [#1] (0,0) rectangle (0.2,0.2);}%
\newcommand{\tikzcircle}[2][red,fill=red]{\tikz[baseline=-0.5ex]\draw[#1,radius=#2] (0,0) circle ;}%
\newcommand{\argmin}{\mathop{\mbox{argmin}}}

\newcommand{\subparagraph}{}
\usepackage{titlesec}
\titlespacing{\section}{8pt}{7pt}{6pt}
\newcommand{\cfbox}[2]{%
    \colorlet{currentcolor}{.}%
    {\color{#1}%
    \fbox{\color{currentcolor}#2}}%
}
\newcommand{\costplots}{$\sum_{n=0}^{N-1}\left\Vert \mathbf{j}_{n}\right\Vert ^{2}dt + \rho T$, where $\rho=0.2$~m$^2$/s$^6$}
\usepackage[sort,compress]{cite}

\usepackage{scalerel}
\usetikzlibrary{svg.path}

\definecolor{orcidlogocol}{HTML}{A6CE39}
\tikzset{
	orcidlogo/.pic={
		\fill[orcidlogocol] svg{M256,128c0,70.7-57.3,128-128,128C57.3,256,0,198.7,0,128C0,57.3,57.3,0,128,0C198.7,0,256,57.3,256,128z};
		\fill[white] svg{M86.3,186.2H70.9V79.1h15.4v48.4V186.2z}
		svg{M108.9,79.1h41.6c39.6,0,57,28.3,57,53.6c0,27.5-21.5,53.6-56.8,53.6h-41.8V79.1z M124.3,172.4h24.5c34.9,0,42.9-26.5,42.9-39.7c0-21.5-13.7-39.7-43.7-39.7h-23.7V172.4z}
		svg{M88.7,56.8c0,5.5-4.5,10.1-10.1,10.1c-5.6,0-10.1-4.6-10.1-10.1c0-5.6,4.5-10.1,10.1-10.1C84.2,46.7,88.7,51.3,88.7,56.8z};
	}
}

\newcommand\orcidicon[1]{\href{https://orcid.org/#1}{\mbox{\scalerel*{
				\begin{tikzpicture}[yscale=-1,transform shape]
					\pic{orcidlogo};
				\end{tikzpicture}
			}{|}}}}

\begin{document}
\title{\add{FASTER: Fast and Safe Trajectory Planner for Navigation in Unknown Environments}}

\author{Jesus~Tordesillas$^1$, %
        Brett~T.~Lopez$^2$, %
        Michael~Everett$^1$, %
        and~Jonathan~P.~How$^1$ %
\thanks{$^1$The authors are with the Aerospace Controls Laboratory, MIT, 77 Massachusetts Ave., Cambridge, MA, USA \tt\{jtorde, mfe, jhow\}@mit.edu}%
\thanks{$^2$The author is with the Jet Propulsion Laboratory, California Institute of Technology, 4800 Oak Grove Dr. Pasadena, CA. \tt brett.t.lopez@jpl.nasa.gov     }%
\thanks{Manuscript received in July 2020; revised XXXX, 2020.}}

\markboth{ }%
{Shell \MakeLowercase{\textit{et al.}}: Bare Demo of IEEEtran.cls for IEEE Journals}

\maketitle

\begin{tikzpicture}[overlay, remember picture]
	\path (current page.north) ++(0.0,-1.0) node[draw = black] {\small This paper has been accepted for publication in \textit{IEEE Transactions on Robotics}};
\end{tikzpicture}
\vspace{-0.3cm}

\begin{abstract}
Planning high-speed trajectories for UAVs in unknown environments requires algorithmic techniques that enable fast reaction times to guarantee safety as more information about the environment becomes available. The standard approaches that ensure safety by enforcing a ``stop" condition in the free-known space can severely limit the speed of the vehicle, especially in situations where much of the world is unknown. Moreover, the ad-hoc time and interval allocation scheme usually imposed on the trajectory also leads to conservative and slower trajectories. This work proposes FASTER (Fast and Safe Trajectory Planner) to ensure safety without sacrificing speed. FASTER obtains high-speed trajectories by enabling the local planner to optimize in both the free-known and unknown spaces. Safety is ensured by always having a safe back-up trajectory in the free-known space. The MIQP formulation proposed also allows the solver to choose the trajectory interval allocation. FASTER is tested extensively in simulation and in real hardware, showing flights in unknown cluttered environments with velocities up to $7.8$ m/s, and experiments at the maximum speed of a skid-steer ground robot ($2$ m/s).
\end{abstract}


\IEEEpeerreviewmaketitle

{\footnotesize

\textbf{Acronyms}: UAV~(Unmanned Aerial Vehicle), MIQP~(Mixed-Integer Quadratic Program), RRT~(Rapidly-Exploring Random Tree), VIO~(Visual-Inertial Odometry), FOV~(Field of view).\\

\textbf{Code}: 
\begin{itemize}
 \item FASTER: \href{https://github.com/mit-acl/faster}{https://github.com/mit-acl/faster}
 \item Simulation worlds: \url{https://github.com/jtorde/planning_worlds_gazebo}
\end{itemize}

\textbf{Video}: \href{https://youtu.be/fkkkgomkX10}{https://youtu.be/fkkkgomkX10}
}

\section{Introduction}\label{sec:introduction}
\IEEEPARstart{D}{espite} its numerous applications, high-speed UAV navigation through unknown environments is still an open problem. 
The desired high speeds together with partial observability of the environment and limits on payload weight make this task especially challenging for aerial robots.
Safe operation, in addition to flying fast, is also critical but difficult to guarantee since the vehicle must repeatedly generate collision-free, dynamically feasible trajectories in real-time with limited sensing. Similar to the model predictive control literature, safety is guaranteed by ensuring a feasible solution exists indefinitely.

\begin{figure}[t]
	\centering
	\includegraphics[width=\columnwidth]{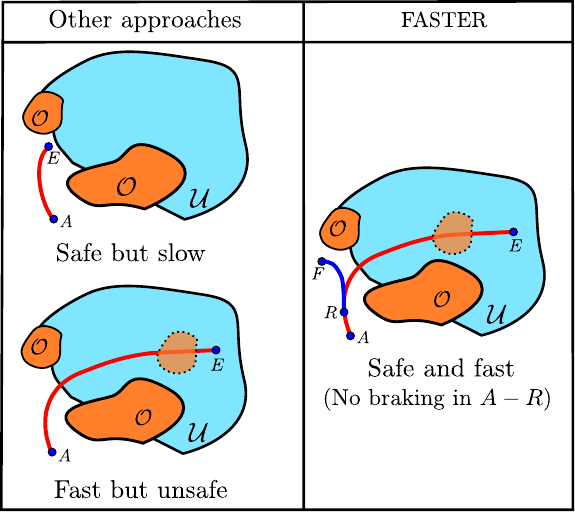}
	\caption{Safety and Speed tradeoff. $\mathcal{O}$ is the occupied-known space~(\tikzrectangle[black,fill=DarkOrange]{10pt}), and $\mathcal{U}$ is the unknown space (\tikzrectangle[black,fill=LightBlue]{10pt}). $A$ and $E$ are, respectively, the start and goal locations of the local plan.}
	\label{fig:contributions_part1}
	\vspace{-0.5cm}
\end{figure} 

If we consider $\mathbb{R}^3=\mathcal{F}\cup\mathcal{O}\cup\mathcal{U} $ where $\mathcal{F}$, $\mathcal{O}$, $\mathcal{U}$ are disjoint sets denoting free-known, occupied-known, and unknown space respectively, the following hierarchical planning architecture is commonly used: a global planner first finds the shortest piece-wise linear path from the UAV to the goal, avoiding the known obstacles $\mathcal{O}$. Then, a local planner finds a dynamically feasible trajectory in the direction given by this global plan. This local planner should find a fast \textit{and} Safe Trajectory that leads the UAV to the goal. These two requirements of \textbf{safety} and \textbf{speed} represent the following tradeoff: on one hand, safety argues for short trajectories completely contained in $\mathcal{F}$ and end points not necessarily near the global plan. As a final stop condition is needed to guarantee safety, short trajectories are generally much slower than long trajectories because the braking maneuver propagates backwards from the end to the initial state of the trajectory. On the other hand, speed argues for longer planned trajectories (usually extending farther than $\mathcal{F}$) and end points near the global plan. 

The typical way to solve the speed versus safety tradeoff is to ensure safety by planning only in $\mathcal{F}$, and then impose a final stop condition near the global plan. This can be achieved by either generating motion primitives that do not intersect $\mathcal{O} \cup \mathcal{U}$ \cite{mueller2015computationally,lopez2017aggressive3D,lopez2017aggressivelimitedFOV, tordesillas2018real}, or by constructing a convex representation of $\mathcal{F}$ to be used in an optimization \cite{deits2015efficient, liu2017planning, preiss2017trajectory}. The main limitation of these works is that safety is guaranteed at the expense of higher speeds, especially in scenarios where $\mathcal{F}$ is small compared to $\mathcal{O} \cup \mathcal{U}$. This article presents an optimization-based approach that solves this limitation by solving for \textit{two} optimal trajectories at every planning step (see Fig.~\ref{fig:contributions_part1}): The \textcolor{red}{first trajectory} is in $\mathcal{U}\cup\mathcal{F}$ and ensures a long planning horizon with an end point on the global plan. The \textcolor{blue}{second trajectory} is in $\mathcal{F}$, starts from a point along the first trajectory, and it may deviate from the global plan. Only a portion of the first trajectory is actually implemented by the UAV (therefore satisfying the \textbf{speed} requirement), while the second trajectory guarantees \textbf{safety}, since it is contained in $\mathcal{F}$ and available at the start of every replanning step. This second trajectory is only implemented if the optimization problem becomes infeasible in the next replanning steps.

A second limitation, specially for the optimization-based approaches that use convex decomposition, is the choice of the interval and time allocation method. The interval allocation decides in which polyhedron each interval of the trajectory will be located, whereas the time allocation deals with the time spent on each interval (see Fig.~\ref{fig:contributions_part2}). In order to simplify the interval allocation, a common choice is to set the number of intervals to be the same as the number of polyhedra found, forcing each interval to be in one specific polyhedron. This forces the optimizer to select the end points of each trajectory segment within the overlapping area of two consecutive polyhedra, and therefore possibly leading to more conservative or longer trajectories. 
 Moreover, since a different time for each interval has to be found, the time allocation calculation is harder, leading to higher replanning times when using optimization techniques to allocate this time, and to nonsmooth or infeasible trajectories when imposing an ad-hoc time allocation. To overcome this limitation, FASTER allows the solver to decide the interval allocation by using a number of intervals greater than the number of polyhedra found \cite{landry2016aggressive} and by allocating the same time for all the intervals. This time allocation method is efficiently found through a line search algorithm initialized with the solution at the previous replanning iteration.

\begin{figure}[t]
	\centering
	\includegraphics[width=\columnwidth]{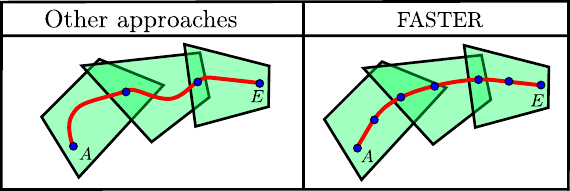}
	\caption{Interval and Time Allocation when using a convex decomposition (\tikzrectangle[black,fill=MylightGreen]{10pt}). $A$ and $E$ are, respectively, the start and goal locations of the local plan.}
	\label{fig:contributions_part2}
\end{figure}

The planning framework proposed is called \textbf{FASTER} - \textbf{FA}st and \textbf{S}afe \textbf{T}rajectory Plann\textbf{ER}, and is an extension of our two published conference papers \cite{tordesillas2018real, tordesillas2019faster}.
In summary, this work has the following contributions:
\begin{itemize}
  \item A framework that ensures feasibility of the entire collision avoidance algorithm and guarantees safety without reducing the nominal flight speed by allowing the local planner to plan in $\mathcal{F} \cup \mathcal{U}$  while always having a Safe Trajectory in $\mathcal{F}$.
	\item Reduced conservatism of the time and interval allocation compared to prior ad-hoc approaches by efficiently finding the time allocated from the result of the previous replanning iteration and then allowing the optimizer to choose the interval allocation.
	\item Extension of our previous work \cite{tordesillas2018real} by proposing a way to compute very cheaply a heuristic of the cost-to-go needed by the local planner to decide which direction is the best one to optimize toward.
	\item Simulation and hardware experiments showing agile flights in completely unknown cluttered environments, with velocities up to $7.8$ m/s, two times faster than previous state-of-the-art methods \cite{tordesillas2019faster,tordesillas2018real}. FASTER is also tested on a skid-steer robot, showing hardware experiments at the top speed of the robot ($2$ m/s).

\end{itemize}

\add{
In particular, the new contributions of this version with respect to the conference papers \cite{tordesillas2018real, tordesillas2019faster} are:
\begin{enumerate}[a)] %
\item {Theoretical analysis}: Feasibility theorem that guarantees safety for FASTER. %
\item {Simulation}: Cluttered office simulation, which presents a major challenge in terms of both clutterness for obstacle avoidance and limited visibility. %
\item {Hardware}: Duplication of the flight volume, achieving velocities up to $7.8$ m/s. %
\item {Extension: Skid-steer robot.} %
\end{enumerate}
}

\add{
Moreover, we also perform a deeper analysis of the role of the Safe Trajectory in terms of safety and speed, a comparison of the performance of the interval allocation vs. the time allocation, and a comparison between the flight corridors associated with the safe and whole trajectories.
}

\section{RELATED WORK}\label{sec:related_work}
Different methods have been proposed in the literature for planning, mapping, and the integration of these two (Fig.~\ref{fig:classification}):  

\cfbox{blue}{\textbf{Planning}} for UAVs can be classified according to the specific formulation of the optimization problem and the operating space of the local planner. 

As far as the \textit{\textbf{optimization problem}} itself is concerned, most of the current state-of-the-art methods exploit the differential flatness of the quadrotors, and, using an integrator model, minimize the squared norm of a derivative of the position to find a dynamically feasible smooth trajectory \cite{mellinger2011minimum, van1998real, richter2016polynomial}. When there are obstacles present, 
\add{some methods include them as constraints in an optimization problem, while others take these obstacles into account either after the optimization or during a search-based algorithm}.

There are approaches where the obstacle constraints (and sometimes also the input constraints) are 
\add{either checked after solving the optimization problem, or imposed during a search-based algorithm}: some of them use stitched polynomial trajectories that pass through several waypoints obtained running RRT-based methods \cite{mellinger2011minimum, richter2016polynomial, loianno2017estimation}, while others use closed-form solutions or motion-primitive libraries \cite{mueller2015computationally, florence2016integrated, lopez2017aggressivelimitedFOV, lopez2017aggressive3D, bucki2019rapid, markus2019efficient, spitzer2019fast}. \add{These methods are usually limited to short trajectories unable to perform complex maneuvers around obstacles. Sometimes these primitives are also used to search over the state space \cite{liu2017searchminimumT, liu2018searchBasedSE3, zhou2019robust}, often benefiting from ESDF representations to guide the search. However, the search-based methods are usually computationally expensive, especially in cluttered environments}.

\begin{figure}[t]
	\centering
	\includegraphics[width=0.7\columnwidth]{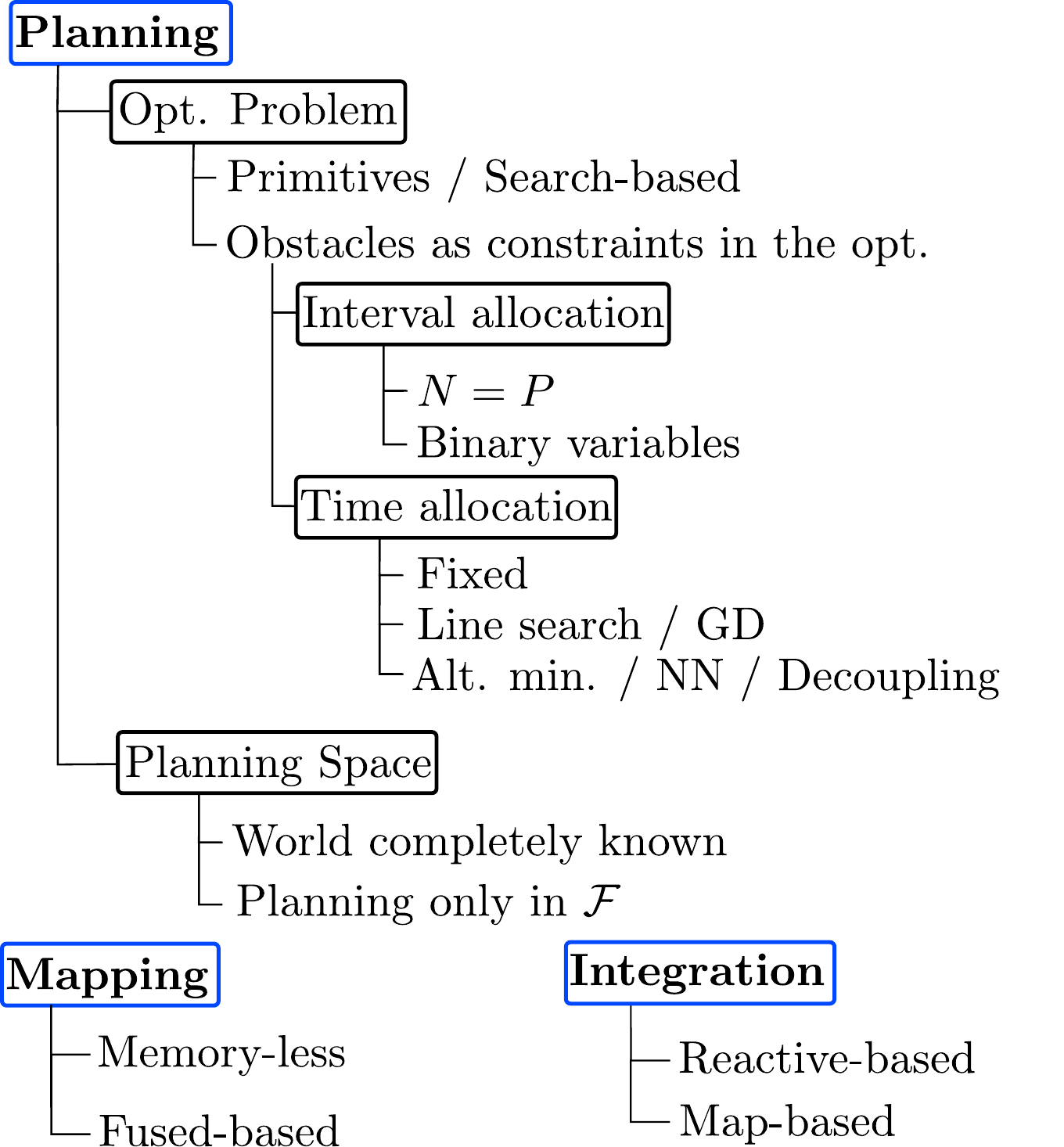}
	\caption{\add{Classification of the state-of-the-art techniques for planning, mapping, and the integration between these two.}}
	\label{fig:classification}
\end{figure}

The other approach is to include the obstacles directly 
\add{as constraints in an} 
optimization problem. This can be done in the cost function by penalizing the distance to the obstacles \cite{oleynikova2016continuous, oleynikova2018safe}, but this usually leads to computationally expensive distance fields representations and/or nonconvex optimization problems. Another option is to encode the shape of the obstacles in the constraints using successive convexification \cite{mao2018successive, augugliaro2012generation, schulman2014motion} or a convex decomposition of the environment \cite{liu2018convex, liu2017planning, wattersontrajectory, gao2019flying, lai2019safe, rousseau2019minimum}. The convergence of successive convexification typically depends on the initial guess, and is usually not suitable for real-time planning in unknown cluttered environments. The convex decomposition approach is usually done by decomposing the free-known space as a series of $P$ overlapping polyhedra \cite{liu2017planning, deits2015efficient, preiss2017trajectory}. As the trajectory is usually decomposed of $N$ third (or higher)-degree polynomials, to guarantee that the Whole Trajectory is inside the polyhedra, B\'{e}zier Curves \cite{preiss2017trajectory, sahingoz2014generation}, or the sum-of-squares condition \cite{deits2015efficient, landry2016aggressive} are often used. Moreover, for a trajectory there is both an interval (in which polyhedron each interval is) and a time allocation (how much time is assigned to each interval) problem. For the \textbf{interval allocation}, a usual decision is to use $N=P$ intervals, and force each interval to be inside its corresponding polyhedron \cite{preiss2017trajectory}. However, this sometimes can be very conservative, since the solver can only choose to place the two extreme points of each interval in the overlapping area of two consecutive polyhedra. Another option, but sometimes with higher computation times, is to use binary variables \cite{landry2016aggressive, deits2015efficient} to allow the solver to choose the specific interval allocation. \add{For the \textbf{time allocation}, different techniques have been proposed. 
One way is to impose a fixed time allocation using a specific velocity profile \cite{liu2017planning}, which can be conservative, or cause infeasibility in the optimization problem if the overlapping area of the polyhedra is not large enough. Other options are to use line search or gradient descent to iteratively obtain these times \cite{richter2016polynomial,mellinger2011minimum, liu2016high}, to use alternating minimization between the spatial and temporal trajectory \cite{wang2020alternating}, or to implement a neural network trained offline \cite{de2019real}. Another option is to decouple the spatial and the temporal trajectory \cite{gao2018optimal}, but, as noted in this work, this may cause infeasibility if the initial and final states are not static.   
 }

With regard to the \textit{\textbf{planning space}} of the local planner, several approaches have been developed. 
One approach is to use only the most recent perception data \cite{lopez2017aggressivelimitedFOV, lopez2017aggressive3D}, which requires the desired trajectory to remain within the perception sensor field of view. 
An alternative strategy is to create and plan trajectories in a map of the environment built using a history of perception data. Within this second category, in some works \cite{schouwenaars2002safe,tordesillas2018real,oleynikova2018safe}, the local planner only optimizes inside $\mathcal{F}$, which guarantees safety if the local planner has a final stop condition. However, limiting the planner to operating in $\mathcal{F}$ and enforcing a terminal stopping condition can lead to conservative, slow trajectories (especially when much of the world is unknown).  
Higher speeds can be obtained by allowing the local planner to optimize in both the free-known and unknown space ($\mathcal{F}\cup\mathcal{U}$), but with no guarantees that the trajectory is safe or will remain feasible. 

Moreover, two main categories can be highlighted in the \cfbox{blue}{\textbf{mapping}} methods proposed in the literature: memory-less and fused-based methods. The first category includes the approaches that rely only on instantaneous sensing data, using only the last measurement, or weighting the data \cite{dey2016vision, florence2018nanomap, gao2019flying, florence2016integrated}. These approaches are in general unable to reason about obstacles observed in the past \cite{lopez2017aggressive3D, lopez2017aggressivelimitedFOV}, and are specially limited when a sensor with small FOV is used. The second category is the fusion-based approach, in which the sensing data are fused into a map, usually in the form of an occupancy grid or distance fields \cite{lau2010improved, oleynikova2017voxblox}. Two drawbacks of these approaches are the influence of the estimation error, and the fusion time.

Finally, several approaches have been proposed for the \cfbox{blue}{\textbf{integration}} between the planner and the mapper: reactive and map-based planners. Reactive planners often use a memory-less representation of the environment, and closed-form primitives are usually chosen for planning \cite{lopez2017aggressive3D, lopez2017aggressivelimitedFOV}. These approaches often fail in complex cluttered scenarios. On the other hand, map-based planners usually use occupancy grids or distance fields to represent the environment. These planners either plan all the trajectory at once or implement a receding horizon planning framework, optimizing trajectories locally and based on a global planner. Moreover, when unknown space is also taken into consideration, several approaches are possible: some use optimistic planners that consider unknown space as free \cite{pivtoraiko2013incremental, chen2016online}, while in other works an optimistic global planner is used combined with a conservative local planner \cite{oleynikova2016continuous, oleynikova2018safe}.

\section{FASTER}  \label{sec:algorithm}
The notation used throughout this article is shown in Fig.~\ref{fig:notation}: $\mathcal{M}$ is a sliding map centered on $L$, the current position of the UAV. $\mathcal{F}$ and $\mathcal{O}$ will denote the free-known and occupied-known spaces respectively. Similarly, $\mathcal{F}_{\text{Unknown}}$ and $\mathcal{O}_{\text{Unknown}}$ will denote the free-unknown and occupied-unknown spaces, respectively. The total unknown space, denoted as $\mathcal{U}$, is therefore $\mathcal{U}=\mathcal{F}_\text{{Unknown}}\cup\mathcal{O}_{\text{Unknown}}$, and $\mathcal{F}$ and $\mathcal{O}$ are completely contained inside the map ($\mathcal{F} \cup \mathcal{O}	\subseteq \mathcal{M}$), and all the space outside the map is inside $\mathcal{U}$ ($ \mathbb{R}^3 \setminus  \mathcal{M} \subseteq \mathcal{U} $). Note also that FASTER is completely in 3-D, but some illustrations are in 2-D for visualization purposes. 

\begin{figure}[t]
	\centering
	\includegraphics[width=1\columnwidth]{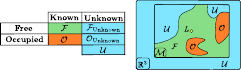}
	\caption{ Notation used for the different spaces. $L$ is the current position of the UAV, and $\mathcal{M}$ is the sliding map around the vehicle. }
	\label{fig:notation}
\end{figure}

\subsection{Mapping}
A body-centered sliding map $\mathcal{M}$ (in the form of an occupancy grid map) is used in this work.
A rolling map is desirable since it reduces the influence of the drift in the estimation error. We fuse a depth map into the occupancy grid using the 3-D Bresenham's line algorithm for ray-tracing \cite{bresenham1965algorithm}. Both $\mathcal{O}$ and $\mathcal{U}$ are inflated by the radius of the UAV to ensure safety. 

\subsection{Global Planner}
In the proposed framework, Jump Point Search (JPS)  is used as a global planner to find the shortest piece-wise linear path from the current position to the goal. JPS was chosen instead of A* because it runs an order of magnitude faster, while still guaranteeing completeness and optimality \cite{harabor2011online, liu2017planning}. The only assumption of JPS is a uniform grid, which holds in our case. 

\vspace{0.5cm}
\subsection{Convex Decomposition}
A convex decomposition is done around part of the piece-wise linear path obtained by JPS. To do this convex decomposition, we rely on the approach proposed by \cite{liu2017planning}: A polyhedron is found around each segment of the piece-wise linear path by first inflating an ellipsoid aligned with the segment, and then computing the tangent planes at the points of the ellipsoid that are in contact with the obstacles. The reader is referred to \cite{liu2017planning} for a detailed explanation. Given a piece-wise linear path with $P$ segments, we will denote the sequence of $P$ overlapping polyhedra as $\{(\bb{A}_p,  \bb{c}_p)\},\;p=0:P-1$. 

\vspace{0.5cm}
\subsection{Local Planner}\label{subsec:local_planner}
For the local planner, we distinguish these three different jerk-controlled trajectories (see Fig.~\ref{fig:different_trajectories}):
\begin{itemize}
	\item \textbf{\textcolor{red}{Whole Trajectory}:} This trajectory goes from $A$ to $E$, and it is contained in $\mathcal{F}\cup\mathcal{U}$. It has a final stop condition.
	\item \textbf{\textcolor{blue}{Safe Trajectory}:} It goes from $R$ to $F$, where $R$ is a point in the Whole Trajectory, and $F$ is any point inside the polyhedra obtained by doing a convex decomposition of $\mathcal{F}$. It is completely contained in $\mathcal{F}$, and it also has a final stop condition to guarantee safety.
	\item \textbf{\textcolor{ForestGreen}{Committed Trajectory}:} This trajectory consists of two pieces: The first part is the interval $A \rightarrow R$ of the Whole Trajectory. The second part is the Safe Trajectory. It will be shown later that this trajectory is also guaranteed to be inside $\mathcal{F}$. This trajectory is the one that the UAV will keep executing in case no feasible solutions are found in the next replanning steps. 
\end{itemize}

\begin{figure}[t]
	\centering
	\includegraphics[width=0.8\columnwidth]{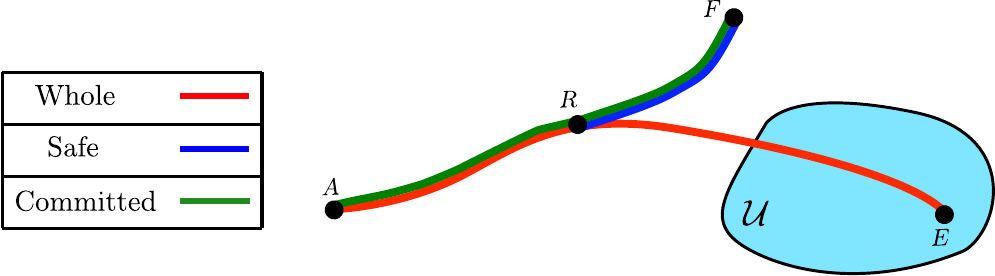}
	\caption{Trajectories used by FASTER: The Committed and Safe Trajectories are inside $\mathcal{F}$, while the Whole Trajectory is inside $\mathcal{F}\cup\mathcal{U}$.}
	\label{fig:different_trajectories}
	\vspace{0.6cm}
\end{figure} 

\begin{figure}[t]
	\centering
	\includegraphics[width=0.8\columnwidth]{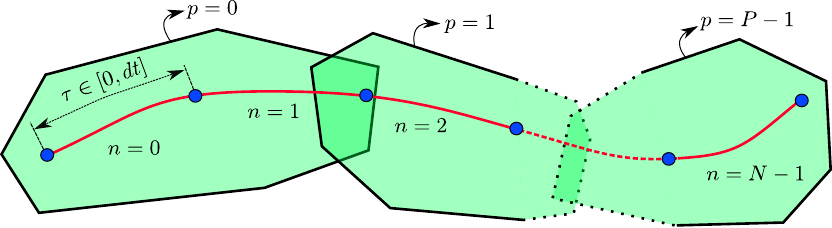}
	\caption{Each interval $n=0:N-1$ of the trajectory is a third-degree polynomial, with a total time of $dt$ per interval. $\tau\in[0,dt]$ denotes a local reference of the time inside an interval, and $p=0:P-1$ denotes the polyhedron.}
	\label{fig:traj_explanation}
	\vspace{0.6cm}
\end{figure}

The quadrotor is modeled using triple integrator dynamics with state vector $
\mathbf{x}^T = \left[ \vect{x}^T ~ \dvect{x}^T ~ \ddvect{x}^T ~ \right] = \left[\vect{x}^T ~\vect{v}^T ~\vect{a}^T\right]
$
and control input $\mathbf{u} = \dddvect{x} = \vect{j}$ (where $\vect{x}$, $\vect{v}$, $\vect{a}$, and $\vect{j}$ are the vehicle's position, velocity, acceleration, and jerk, respectively).

In the optimization problem solved by the local planner, the trajectory is divided in $N$ intervals (see Fig.~\ref{fig:traj_explanation}). Let $n=0:N-1$ denote the specific interval of the trajectory, $p=0:P-1$ the specific polyhedron and $dt$ the time allocated per interval (same for every interval $n$). If $\bb{j}(t)$ is constrained to be constant in each interval $n=0:N-1$, then the Whole Trajectory will be a spline consisting of third-degree polynomials. Matching the cubic form of the position for each interval $$\boldsymbol{x}_{n}(\tau)=\mathbf{a}_{n}\tau^{3}+\mathbf{b}_{n}\tau^{2}+\mathbf{c}_{n}\tau+\mathbf{d}_{n},\; \; \tau\in[0,dt]$$ with the expression of a cubic B\'{e}zier curve
$$\boldsymbol{x}_{n}(\tau)=\sum_{j=0}^{3}\left(\begin{array}{c}
3\\
j
\end{array}\right)\left(1-\frac{\tau}{dt}\right)^{3-j}\left(\frac{\tau}{dt}\right)^{j}\boldsymbol{r}_{nj},\;\; \tau\in[0,dt],$$
we can solve for the four control points  $\bb{r}_{nj}$ $(j=0:3)$ associated with each interval $n$:
\begin{align} & \boldsymbol{r}_{n0}=\mathbf{d}_{n}, \qquad 
	\boldsymbol{r}_{n1}=\frac{\mathbf{c}_{n}dt+3\mathbf{d}_{n}}{3} & \nonumber\\
	& \boldsymbol{r}_{n2}=\frac{\mathbf{b}_{n}dt^{2}+2\mathbf{c}_{n}dt+3\mathbf{d}_{n}}{3} & \nonumber\\
	& \boldsymbol{r}_{n3}=\mathbf{a}_{n}dt^{3}+\mathbf{b}_{n}dt^{2}+\mathbf{c}_{n}dt+\mathbf{d}_{n}  & \nonumber
\end{align}

Let us introduce the binary variables $b_{np}$, with $p=0:P-1$ and $n=0:N-1$ ($P$ variables for each interval $n=0:N-1$). As a B\'{e}zier curve is contained in the convex hull of its control points, we can ensure that the trajectory will be completely contained in this convex corridor by forcing that all the control points of an interval $n$ are in the same polyhedron \cite{sahingoz2014generation, preiss2017trajectory} with the constraint $[b_{np}=1\implies \boldsymbol{r}_{nj} \in \text{polyhedron $p$} \;\; \forall j]$, and at least in one polyhedron with the constraint $\sum_{p=0}^{P-1}b_{np}\ge1$. With this formulation, the optimizer is free to choose the specific interval allocation (i.e., which interval is inside which polyhedron).
The complete MIQP solved in each replanning step for both the Safe and the Whole trajectories is as follows:
\begin{align}\min_{\boldsymbol{j}_{n}, b_{np}} & \sum_{n=0}^{N-1}\left\Vert \mathbf{j}_{n}\right\Vert ^{2}dt\label{eq:MIQP}\\
	\textrm{s.t. } & \mathbf{x}_{0}(0)=\mathbf{x}_{\text{init}}\nonumber\\
	& \mathbf{x}_{N-1}(dt)=\mathbf{x}_{\text{final}}\nonumber\\
	& \boldsymbol{x}_{n}(\tau)=\mathbf{a}_{n}\tau^{3}+\mathbf{b}_{n}\tau^{2}+\mathbf{c}_{n}\tau+\mathbf{d}_{n}\;\forall n,\forall\tau\in[0,dt] & \nonumber\\
	& \boldsymbol{v}_{n}(\tau)=\dot{\boldsymbol{x}}_{n}(\tau)\;\qquad\qquad\qquad\qquad\;\forall n,\forall\tau\in[0,dt]\nonumber\\
	& \boldsymbol{a}_{n}(\tau)=\dot{\boldsymbol{v}}_{n}(\tau)\;\qquad\qquad\qquad\qquad\;\forall n,\forall\tau\in[0,dt]\nonumber\\
	& \mathbf{j}_{n}=6\mathbf{a}_{n}\;\forall n & \nonumber\\
	& b_{np}=1\implies\left\{ \begin{array}{c} 
		\boldsymbol{A}_{p}\boldsymbol{r}_{n0}\le\boldsymbol{c}_{p}\\
		\boldsymbol{A}_{p}\boldsymbol{r}_{n1}\le\boldsymbol{c}_{p}\\
		\boldsymbol{A}_{p}\boldsymbol{r}_{n2}\le\boldsymbol{c}_{p}\\
		\boldsymbol{A}_{p}\boldsymbol{r}_{n3}\le\boldsymbol{c}_{p}
	\end{array}\right.\quad\forall n,\forall p & \nonumber\\
	& \sum_{p=0}^{P-1}b_{np}\ge1\qquad\qquad \quad \;\;\; \forall n &\nonumber\\
	& b_{np}\in\{0,1\}\qquad \qquad \qquad  \forall n,\forall p & \nonumber\\
	& \mathbf{x}_{n+1}(0)=\mathbf{x}_{n}(dt)\quad \qquad \;\; n=0:N-2 & \nonumber\\
	& \left\Vert \boldsymbol{v}_{n}(0)\right\Vert _{\infty}\le v_{\text{max}} \qquad\quad\;\forall n & \nonumber\\
	& \left\Vert \boldsymbol{a}_{n}(0)\right\Vert _{\infty}\le a_{\text{max}}\qquad\quad\;\forall n & \nonumber\\
	& \left\Vert \boldsymbol{j}_{n}\right\Vert _{\infty}\le j_{\text{max}}\qquad\quad\quad\;\;\;\forall n & \nonumber
\end{align}

This problem is solved using Gurobi \cite{gurobi}. The decision variables of this optimization problem are the binary variables $b_{np}$ and the jerk along the trajectory $\boldsymbol{j}_{n}$. $\mathbf{x}_{\text{init}}$ and  $\mathbf{x}_{\text{final}}$ denote the initial and final states of the trajectory, respectively. The time $dt$ allocated per interval is computed as:
\begin{equation}\label{eq: find_dt}
    \resizebox{0.91\hsize}{!}{%
        $dt=f\cdot\max\{T_{v_x},T_{v_y}, T_{v_z}, T_{a_x},T_{a_y},T_{a_z},T_{j_x},T_{j_y},T_{j_z}\}/N$%
        }
\end{equation}
where $T_{v_i}$, $T_{a_i}$, $T_{j_i}$ are solution of the constant-input motions in each axis $i=\{x,y,z\}$ by applying $v_{\text{max}}$, $a_{\text{max}}$ and $j_{\text{max}}$, respectively. $f \ge 1$ is a factor that is obtained according to the solution of the previous replanning step (see Fig.~\ref{fig:dynamic_adaptation_factor}): Denoting $f_{\text{worked},k-1}$ as the factor that made the optimization feasible in the replanning step $k-1$, in the replanning step $k$ the optimizer will try values of $f$ (in increasing order) in the interval  $[f_{\text{worked},k-1}-\gamma,f_{\text{worked},k-1}+\gamma']$ until the problem converges. Here, $\gamma$ and $\gamma'$ are constant values chosen by the user. Note that, if $f=1$, then $dt$ is a lower bound on the minimum time per interval required for the problem to be feasible. Therefore, only factors $f\ge 1$ are tried. 
\add{This approach is essentially a line search for the time allocation, with the goal of trying to obtain the smallest $dt$ that makes the optimization feasible (leading therefore to faster trajectories), but at the same time trying to minimize the number of trials with different $dt$ needed until convergence.} 

\begin{figure}[]
	\centering
	\includegraphics[width=0.9\columnwidth]{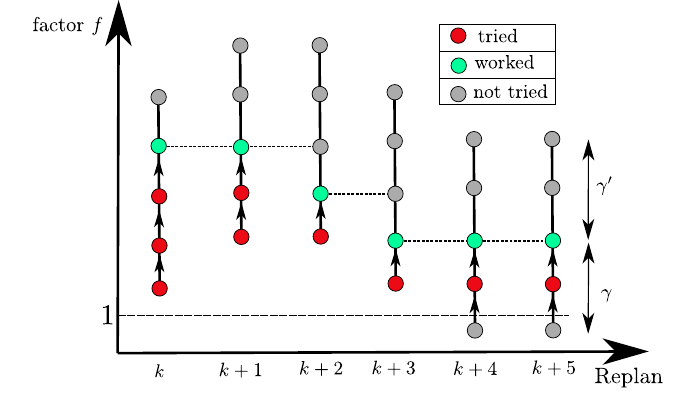}
	\caption[Dynamic adaptation of the factor $f$ used to compute the heuristic of $dt$]{Dynamic adaptation of the factor used to compute the heuristic of the time allocated per interval ($dt$): For iteration $k$, the range of factors used is taken around the factor that worked in the iteration $k-1$. As $f=1$ is the lower bound that makes the problem feasible, only factors $f\ge1$ are tried. }
	\label{fig:dynamic_adaptation_factor}
\end{figure} 

\begin{algorithm}[t]
	\footnotesize
	
	\DontPrintSemicolon
	
	\SetKwFunction{FMain}{\textbf{Replan}}
	\SetKwProg{Pn}{Function}{:}{\KwRet}
	\Pn{\FMain{}}{
		
		$k\leftarrow k+1 $, $\delta t\leftarrow\alpha \Delta t_{k-1}$  \;
Choose point $A$ in $\textcolor{ForestGreen}{\text{Committed}}_{k-1}$ with offset $\delta t$  from $L$\;
$G\leftarrow $ Projection of $G_{\text{term}}$ into map $\mathcal{M}$ \label{projection}\;
$\text{JPS}_a\leftarrow$ Run JPS $A \rightarrow G$\label{points_end}\;
$\mathcal{S} \leftarrow$ Sphere of radius $r$ centered on $A$\;
$C\leftarrow \text{JPS}_a \cap \mathcal{S}$, $\;\;D\leftarrow \text{JPS}_{k-1} \cap \mathcal{S}$
\label{cost_start}\;
\If{$\angle CAD>\alpha_0$}{
$\text{JPS}_b\leftarrow$ Modified $\text{JPS}_{k-1}$ such that $\text{JPS}_{k-1} \cap \mathcal{O}=\emptyset$\;
$D\leftarrow \text{JPS}_{b} \cap \mathcal{S}$\;
$dt_{a} \leftarrow$  Lower bound on $dt$ $A \rightarrow C$\;
$dt_{b} \leftarrow$  Lower bound on $dt$ $A \rightarrow D$\;
$J_a=N \cdot dt_a+\frac{ \left\Vert \text{JPS}_a(C \rightarrow G) \right\Vert }{v_{\text{max}}}$ \;
$J_b=N \cdot dt_b+ \frac{ \left\Vert \text{JPS}_b(D \rightarrow G) \right\Vert }{v_{\text{max}}}$\;
$\text{JPS}_k\leftarrow \underset{\{\text{JPS}_{a},\text{JPS}_{b}\}}{\argmin}\{J_{a},J_{b}\}$\;

}\Else{

$\text{JPS}_k\leftarrow \text{JPS}_a$ \label{cost_end}
}

$\text{JPS}_{\text{in}} \leftarrow $ Part of $\text{JPS}_k$ inside $\mathcal{S}$\label{whole_start}\;
$\text{Poly}_{\text{Whole}} \leftarrow $ Convex Decomposition in $\mathcal{U}\cup\mathcal{F}$ using $\text{JPS}_{\text{in}}$\;
$f_{\text{Whole}} \leftarrow [f_{\text{Whole},k-1}-\gamma,  f_{\text{Whole},k-1}+\gamma' ] $\;
$\textcolor{red}{\text{Whole}} \leftarrow$ MIQP in $\text{Poly}_{\text{Whole}}$ from $A$ to $E$  using $f_{\text{Whole}}$\label{whole_end}\;
$H\leftarrow \textcolor{red}{\text{Whole}} \cap \mathcal{U} \label{safe_start}$\;
$R\leftarrow$ Nearest state to $H$ along \textcolor{red}{Whole} that is not in inevitable collision with $\mathcal{U}$\; 
$\text{JPS}_{\text{in,known}} \leftarrow $ Part of $\text{JPS}_{\text{in}}$ in $\mathcal{F}$\;
$\text{Poly}_{\text{Safe}} \leftarrow $Convex Decomposition in  $\mathcal{F}$ using $\text{JPS}_{\text{in,known}}$\;
$f_{\text{Safe}} \leftarrow [f_{\text{Safe},k-1} -\gamma, f_{\text{Safe},k-1} +\gamma' ] $\;
\textcolor{blue}{$\text{Safe}$} $\leftarrow$ MIQP in $\text{Poly}_{\text{Safe}}$ from $R$ to $F$ using $f_{\text{Safe}}$\label{safe_end}\;

$\textcolor{ForestGreen}{\text{Committed}}_k \leftarrow \textcolor{red}{\text{Whole}}_{A \rightarrow R}\cup \textcolor{blue}{\text{Safe}}\label{committed}$\;
$f_{\text{Whole},k}\leftarrow$ Factor that worked for \textcolor{red}{Whole}\;
$f_{\text{Safe},k}\leftarrow$ Factor that worked for \textcolor{blue}{Safe}\;
$\Delta t_{k}\leftarrow$ Total replanning time\;

	}
	\normalsize
	\caption{FASTER \label{IR}}
	\label{algo: myalgorithm_iros}
\end{algorithm}

\begin{figure}[t]
	\centering
	\includegraphics[width=1\columnwidth]{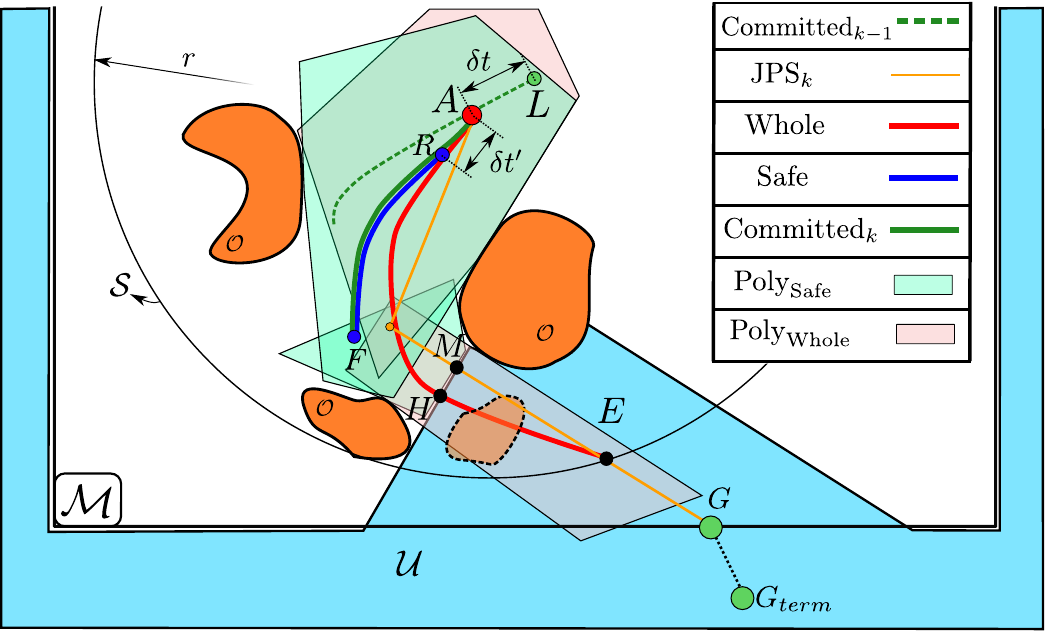}
	\caption[Illustration for Alg.\ref{algo: myalgorithm_iros}]{Illustration for Alg.\ref{algo: myalgorithm_iros}. 
		One unknown obstacle is shown with dotted line.}
	\label{fig:plan2}
\end{figure} 

\begin{figure}[t]
	\centering
	\includegraphics[width=1\columnwidth]{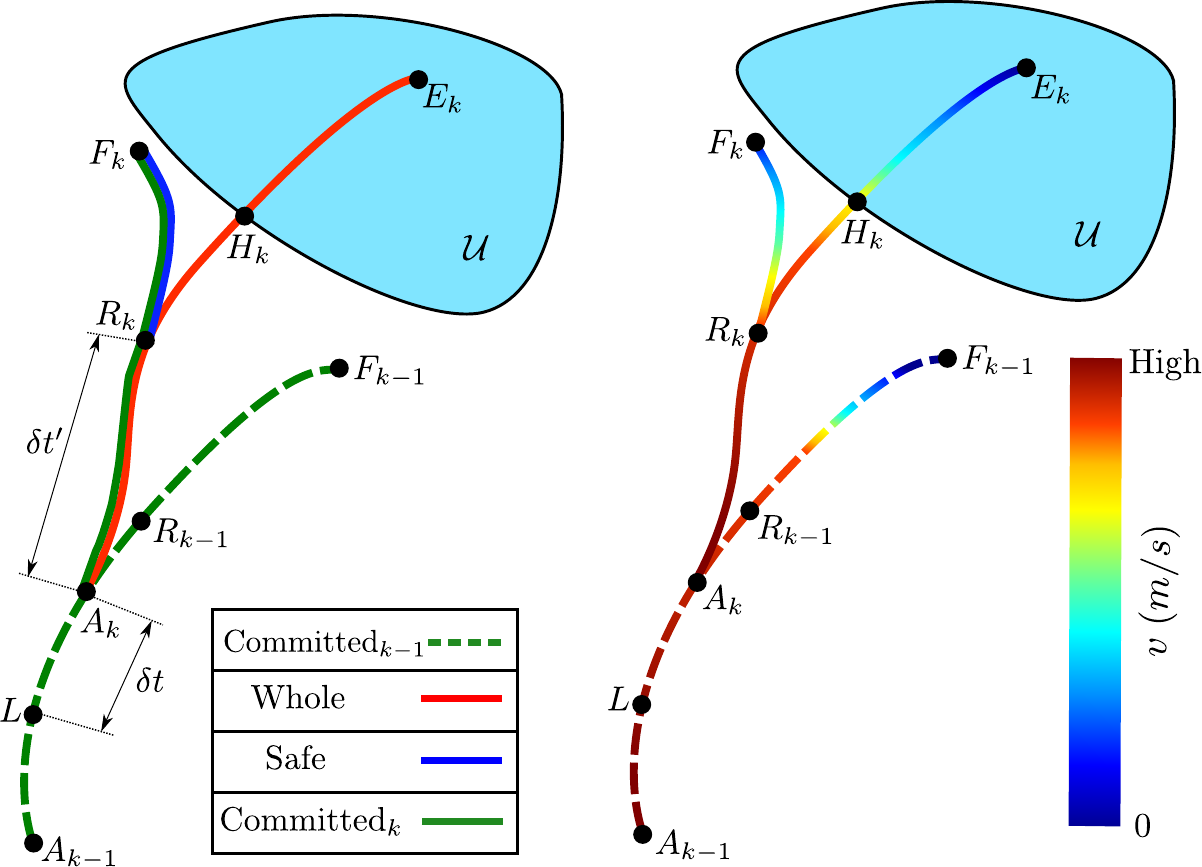}
	\caption[Trajectories used and velocity profiles]{Illustration of all the trajectories involved in Alg. \ref{algo: myalgorithm_iros} and their associated velocity profiles. $\mathcal{U}$ is the unknown space (\tikzrectangle[black,fill=MyblueLight]{10pt}), and $k$ is the replanning step.}
	\label{fig:planning_strategy_with_vel}
\end{figure} 

\definecolor{myPeach}{RGB}{255,158,0}
\definecolor{myPineGreen}{RGB}{22,143,53}
\begin{figure}[]
	\centering
	\includegraphics[width=1\columnwidth]{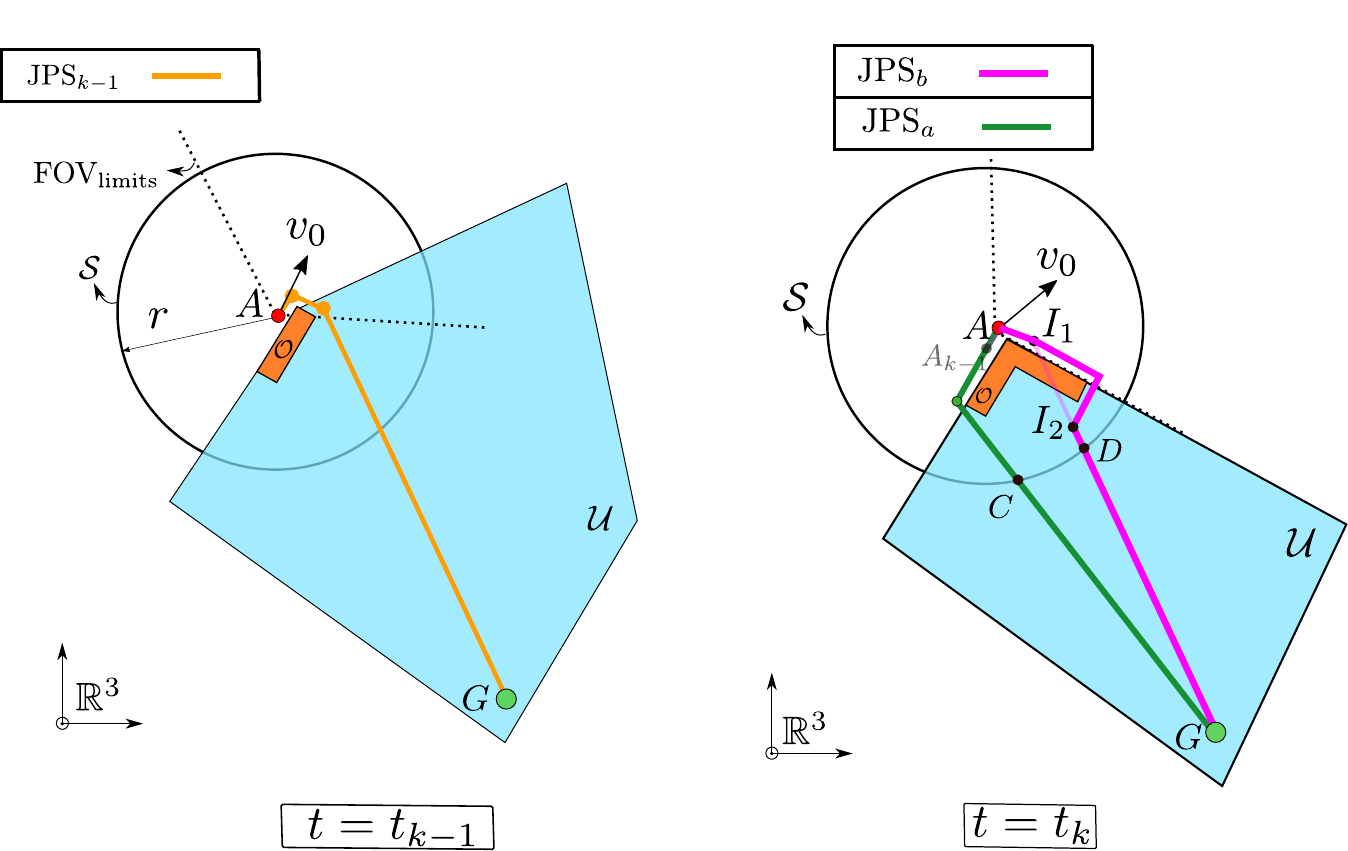}
	\caption[Choice of the direction to optimize]{ Choice of the direction to optimize. At $t=t_{k-1}$, the JPS solution chosen was \textcolor{myPeach}{$\text{JPS}_{k-1}$}. At $t=t_{k}$, JPS is run again to obtain \textcolor{myPineGreen}{$\text{JPS}_a$}, and \textcolor{myPeach}{$\text{JPS}_{k-1}$} is modified so that it does not collide with $\mathcal{O}$, obtaining \textcolor{Magenta}{$\text{JPS}_{b}$}. A heuristic of the cost-to-go in each direction is computed, and the direction with the lowest cost is chosen as the one toward which the local planner will optimize. \add{ By using $A \rightarrow I_1$,  $I_1 \rightarrow I_2$, and  $I_2 \rightarrow G$, \textcolor{Magenta}{$\text{JPS}_{b}$} will pass through the points $A$, $I_1$, $I_2$, and $G$ (all of which belonged to \textcolor{myPeach}{$\text{JPS}_{k-1}$}), and therefore, this gives a close approximation to \textcolor{myPeach}{$\text{JPS}_{k-1}$}, while avoiding $\mathcal{O}$.}}
	\label{fig:planning_strategy1}
\end{figure}

\subsection{Complete Algorithm}\label{subsec:complete_algorithm}

Algorithm \ref{algo: myalgorithm_iros} gives the full approach (see also Figs. \ref{fig:plan2} and \ref{fig:planning_strategy_with_vel}). Let $L$ be the current position of the UAV. The point $A$ is chosen in the Committed Trajectory of the previous replanning step with an offset $\delta t$ from $L$. This offset $\delta t$ is computed by multiplying the total time of the previous replanning step by $\alpha\ge1$ (typically $\alpha \approx 1.25$). The idea here is to dynamically change this offset to ensure that most of the time the solver can find the next solution in less than $\delta t$. Then, the final goal $G_{\text{term}}$ is projected into the sliding map $\mathcal{M}$ (centered on the UAV) in the direction $\overrightarrow{G_{\text{term}}A}$ to obtain the point $G$ (line \ref{projection}). Next, we run JPS from $A$ to $G$ (line \ref{points_end}) to obtain \textcolor{myPineGreen}{$\text{JPS}_a$}.

The local planner then has to decide which direction is the best one to optimize toward (lines \ref{cost_start}-\ref{cost_end}). Instead of blindly trusting the last JPS solution (\textcolor{myPineGreen}{$\text{JPS}_a$}) as the best direction for the local planner to optimize (note that JPS is a zero-order model, without dynamics encoded), we take into account the dynamics of the UAV in the following way: First of all, we modify the \textcolor{myPeach}{$\text{JPS}_{k-1}$} so that it does not collide with the new obstacles seen (Fig.~\ref{fig:planning_strategy1}): we find the points $I_1$ and $I_2$ (first and last intersections of \textcolor{myPeach}{$\text{JPS}_{k-1}$} with $\mathcal{O}$) and run JPS three times, so  $A \rightarrow I_1$,  $I_1 \rightarrow I_2$ and  $I_2 \rightarrow G$. Hence, the modified version, denoted by \textcolor{Magenta}{$\text{JPS}_{b}$}, will be the concatenation of these three paths. \add{Note that by using $A \rightarrow I_1$,  $I_1 \rightarrow I_2$, and  $I_2 \rightarrow G$, we are forcing the combined path to pass through the points $A$, $I_1$, $I_2$, and $G$ (all of which belonged to \textcolor{myPeach}{$\text{JPS}_{k-1}$}), and therefore this gives a close approximation to \textcolor{myPeach}{$\text{JPS}_{k-1}$}, while avoiding $\mathcal{O}$.}

Then, we compute a lower bound on $dt$ using Eq.~\ref{eq: find_dt} for both $A \rightarrow C$ and $A \rightarrow D$, where $C$ and $D$ are the intersections of the previous JPS paths with a sphere $\mathcal{S}$ of radius $r$ centered on $A$, where $r$ is specified by the user. Next, we find the cost-to-go associated with each direction by adding this $dt_a$ (or $dt_b$) and the time it would take the UAV to go from $C$ (or $D$) to $G$ following the JPS solution flying at $v_{\text{max}}$. Finally, the one with lowest cost is chosen, so $\text{JPS}_k\leftarrow \underset{\{\textcolor{myPineGreen}{\text{JPS}_{a}},\textcolor{Magenta}{\text{JPS}_{b}}\}}{\argmin}\{J_{a},J_{b}\}$, \add{which is then the direction toward which the local planner optimizes}.
To save computation time, this decision between \textcolor{myPineGreen}{$\text{JPS}_a$} and \textcolor{Magenta}{$\text{JPS}_{b}$} is made only if the angle $\angle CAD$ exceeds a certain threshold $\alpha_{0}$ (typically $15^{\circ}$). Note that $\angle CAD$ gives a measure of how much the JPS solution has changed with respect to the iteration $k-1$. A small angle indicates that \textcolor{myPineGreen}{$\text{JPS}_a$} and \textcolor{myPeach}{$\text{JPS}_{k-1}$} are very similar (at least within the sphere $\mathcal{S}$), and that therefore the direction of the local plan will not differ much from the iteration $k-1$.  

The \textcolor{red}{Whole Trajectory} (lines \ref{whole_start}-\ref{whole_end}) is obtained as follows. We do the convex decomposition \cite{liu2017planning} of $\mathcal{U} \cup \mathcal{F}$ around the part of $\text{JPS}_{k}$ that is inside the sphere $\mathcal{S}$, which we denote as $\text{JPS}_{\text{in}}$. This gives a series of overlapping polyhedra that we denote as $\text{Poly}_{\text{Whole}}$. Then, the MIQP in (\ref{eq:MIQP}) is solved using these polyhedral constraints to obtain the Whole Trajectory.

The \textcolor{blue}{Safe Trajectory} is computed as in lines \ref{safe_start}-\ref{safe_end}. First, we compute the point $H$ as the intersection between the Whole Trajectory and $\mathcal{U}$. Then, we have to choose the point $R$ along the Whole Trajectory as the start of the Safe Trajectory. To do this, note that, on one hand, $R$ should be chosen as far as possible from $A$, so that $\delta t$ can be chosen larger in the next replanning step, which helps to guarantee that $A$ is not chosen on the Safe Trajectory (where the braking maneuver happens). On the other hand, however, a point $R$ too close to $H$ may lead to an infeasible problem for the Safe Trajectory optimizer. We propose two ways to compute $R$: The first one is to choose it with an offset $\delta t'$ from $A$, where $\delta t'$ is computed by multiplying the previous replanning time by $\beta\ge1$. The second (and better) way to solve this tradeoff is the following one: we can choose $R$ as the nearest state to $H$ (in the segment $A\rightarrow H$ of the Whole Trajectory) that is not in inevitable collision with $\mathcal{U}$. To compute an approximation of this state in a very efficient way, we choose $R$ as the last point (going from $A$ to $H$ along the Whole Trajectory) that satisfies
$$\text{sign}\left[\boldsymbol{v}_{R,j}\left(\boldsymbol{x}_{H,j}-\boldsymbol{x}_{R,j}\right)\right]\cdot\frac{\boldsymbol{v}_{R,j}^{2}}{2\left|a_{\text{max}}\right|}<\left|\boldsymbol{x}_{H,j}-\boldsymbol{x}_{R,j}\right|$$
where $\boldsymbol{v}_{R,j}$, $\boldsymbol{x}_{R,j}$ and $\boldsymbol{x}_{H,j}$ are, respectively, the velocity of $R$, the position of $R$, and the position of $H$ in the axes $j=\{x,y\}$.
Here, we have approximated the system as a double integrator model in each axis and, hence, $\frac{\boldsymbol{v}_{R,j}^{2}}{2\left|a_{\text{max}}\right|}$ is the minimum stopping distance. Due to these two approximations (double integrator and decoupling in axes $x$ and $y$), this heuristic may be conservative. We ignore the axis $z$ in this computation to reduce the conservativeness of this heuristic.

Note that even if this heuristic leads to a choice of $R$ for which no feasible collision-free (with $\mathcal{U}\cup \mathcal{O}$) trajectory exists, the optimizer will not find a solution in that replanning step and, therefore, will continue executing the solution of the previous replanning step. 

After choosing the point $R$, we do the convex decomposition of $\mathcal{F}$ using the part of $\text{JPS}_{\text{in}}$ that is in $\mathcal{F}$, obtaining the polyhedra $\text{Poly}_{\text{Safe}}$. Then, we solve the MIQP from $R$ to any point $F$ inside $\text{Poly}_{\text{Safe}}$ (this point $F$ is chosen by the optimizer).

In both of the convex decompositions presented earlier, one polyhedron is created for each segment of the piecewise linear paths. To obtain a less conservative solution (i.e. larger polyhedra), we first check the length of segments of the JPS path, creating more vertexes if this length exceeds a certain threshold $l_{\text{max}}$. Moreover, we truncate the number of segments in the path to ensure that the number of polyhedra found does not exceed a threshold $P_{\text{max}}$. This helps reduce the computation times (see Sec. \ref{sec:sim_results}).

Finally (line \ref{committed}), we compute the \textcolor{ForestGreen}{Committed Trajectory} by concatenating the piece $A \rightarrow R$ of the Whole Trajectory, and the Safe Trajectory. Note that in this algorithm we have run two \textit{decoupled} optimization problems per replanning step: 1) one for the Whole Trajectory, and 2) one for the Safe Trajectory. This ensures that the piece $A \rightarrow R$ is not influenced by the braking maneuver $R \rightarrow F$, and therefore, it guarantees a higher nominal speed on this first piece. The intervals $L \rightarrow A$ and $A \rightarrow R$ have been designed so that at least one replanning step can be solved within that interval.

The UAV will continue executing the trajectory of the previous replanning step ($\text{Committed}_{k-1}$) if one of these three scenarios happens:
\begin{itemize}
    \item \textbf{Scenario 1:} \label{scenario1} Either of the two optimizations is infeasible. 
    \item \textbf{Scenario 2:} \label{scenario2} The piece $A-R$ intersects $\mathcal{U}$.
    \item \textbf{Scenario 3:} \label{scenario3} The replanning takes longer than $\delta t$.
\end{itemize}

\begin{figure}[]
	\centering
	\includegraphics[width=1\columnwidth]{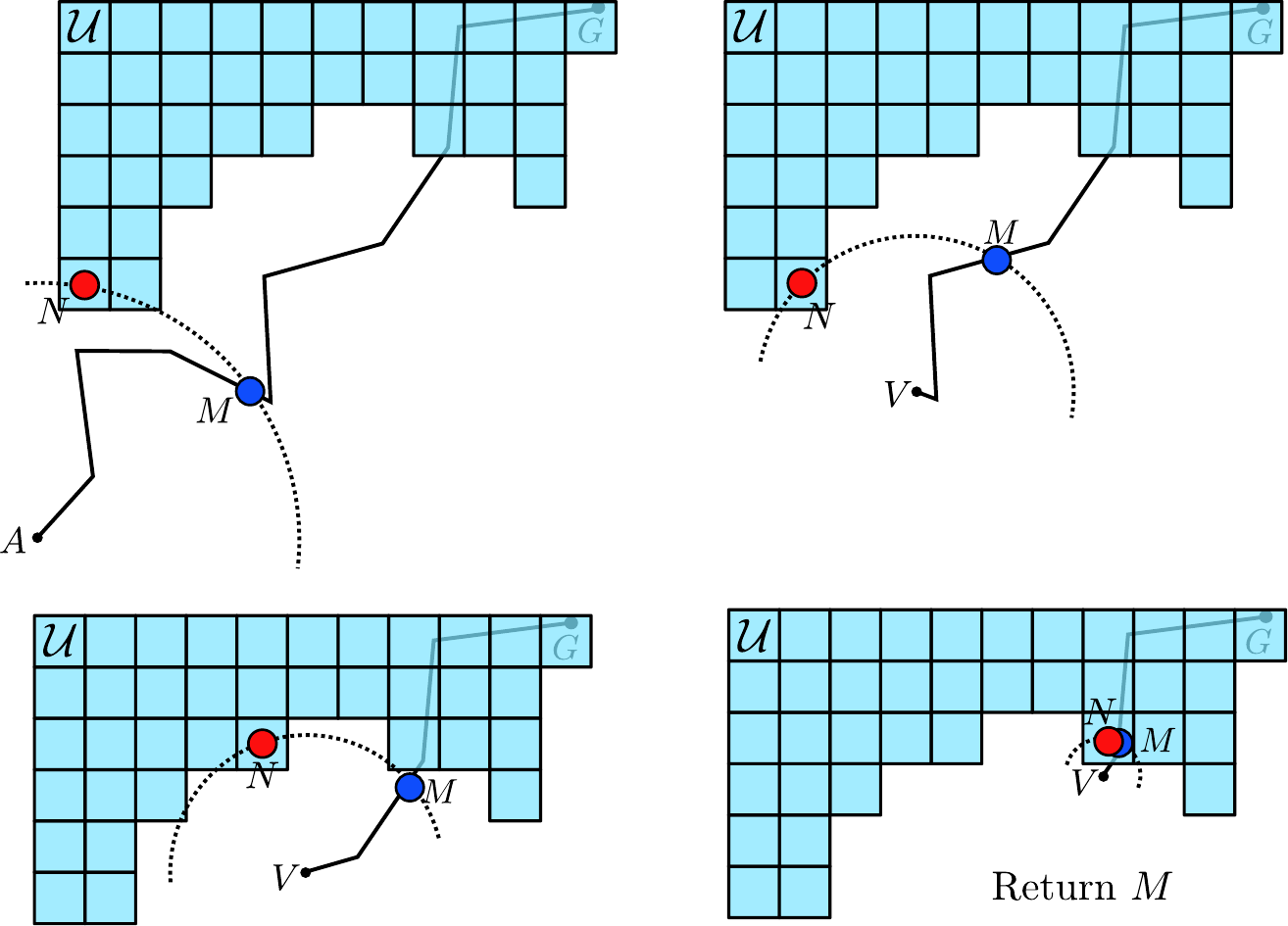}
	\caption[Illustration of the algorithm to find the intersection]{Illustration of Alg. \ref{algo: algorithm_intersection} to efficiently find (an approximation of) the intersection between a piece-wise linear path and a voxel grid. $\mathcal{U}$ and $\text{JPS}_k$ are used in this figure, but in FASTER this algorithm is also used with $\mathcal{O}$ and $\text{JPS}_{k-1}$.}
	\label{fig:find_intersection}
\end{figure} 

\begin{algorithm}[t]
	\footnotesize
	
	\DontPrintSemicolon
	\SetKwFunction{FMain}{\textbf{FindIntersection}}
	\SetKwProg{Pn}{Function}{:}{\KwRet}
	\Pn{\FMain{}}{
		\While{$\text{JPS}_k\neq \emptyset$}{
		    $V\leftarrow$ First element of $\text{JPS}_k$ \; 
			$N \leftarrow$ Find nearest neighbour of $V$ in $\mathcal{U}$ \label{near_neig} \;
			$r \leftarrow \left\Vert N-V \right\Vert$\;
			\If{ $r<\epsilon$  \label{end_inters1}}{
				\Return V \label{end_inters2}
			}			
			$\mathcal{S}\leftarrow$ Sphere of radius $r$ centered on $V$\;
			$M\leftarrow \text{JPS}_k \cap \mathcal{S}  \label{intersec}$\;
			
			Remove from $\text{JPS}_k$ the vertexes inside $\mathcal{S} $\;
			Insert $M$ at the front of $\text{JPS}_k$
		}
		\Return No Intersection
	}
	\normalsize
	\caption{FIND INTERSECTION}
	\label{algo: algorithm_intersection}
\end{algorithm}

In Alg.~\ref{algo: myalgorithm_iros}, it is required to compute the intersection between a piece-wise linear path (the solution of JPS) and a voxel grid ($\mathcal{U}$ or $\mathcal{O}$) to obtain the points $I_1$, $I_2$ or $M$. To do this in an efficient way, we use Alg.~\ref{algo: algorithm_intersection}, depicted in Fig.~\ref{fig:find_intersection}. We first find the nearest neighbor $N$ from the beginning of the piece-wise linear path $A$ (line \ref{near_neig}), and compute the intersection $M$ between the path and a sphere $\mathcal{S}$ centered on $A$ with radius equal to the distance between $A$ and $N$ (line \ref{intersec}). As it is guaranteed that all the points of the path that are inside $\mathcal{S}$ do not intersect with the voxel grid, we can repeat the same procedure again, but this time starting from $M$. This process continues until the distance to the nearest neighbor is below some threshold $\epsilon>0$ (lines \ref{end_inters1}--\ref{end_inters2}). \add{Note that, instead of Alg.~\ref{algo: algorithm_intersection}, another option would be to represent $\mathcal{F}$ as a voxel grid, and then use standard ray-tracing (such as the 3-D Bresenham's line Algorithm \cite{bresenham1965algorithm}) for each of the segments of the  piece-wise linear path. However, this might be very computationally expensive for grids $\mathcal{F}$ with small voxel sizes.}

\subsection{Feasibility Theorem} \label{subsec:feasibilityTheorem}

We can now state the following feasibility theorem for FASTER, which guarantees that all the Committed Trajectories are completely contained inside free space (known or unknown), and that, therefore, safety is guaranteed. Here, $k$ denotes the replanning step. 

\begin{assumption}
\label{assumption_theorem}
The map $\mathcal{M}$ is noise-free and the world is static: $\mathcal{F}_k \cup \mathcal{F}_{\text{Unknown},k}=\mathcal{F}_{k+1} \cup \mathcal{F}_{\text{Unknown},k+1},\;\;\forall k$. 
\end{assumption}

\begin{theorem}
    \label{feasibility_theorem}
	Under the assumption \ref{assumption_theorem}, Alg. \ref{algo: myalgorithm_iros} achieves
	$$\text{Committed}_k \subseteq\mathcal{F}_k \cup \mathcal{F}_{\text{Unknown},k}\;\; \forall k$$
\end{theorem}

\begin{proof}
	This theorem can be proven by induction:
	\begin{enumerate}
		\item \textbf{Base case}: $\text{Committed}_1$ is the union of $A_1\rightarrow R_1$ and  the Safe Trajectory. The interval $A_1\rightarrow R_1$ is in $\mathcal{F}_1$ because it has been checked against collision with $\mathcal{U}_1$ and is contained in a convex corridor that does not intersect $\mathcal{O}_1$. The Safe Trajectory is inside $\mathcal{F}_1$ by construction. Therefore, $\text{Committed}_1\subseteq\mathcal{F}_1 \cup \mathcal{F}_{\text{Unknown},1}$.
		\item \textbf{Recursion}: If  $\text{Committed}_k \subseteq\mathcal{F}_k \cup \mathcal{F}_{\text{Unknown},k}$, two different situations can happen in iteration $k+1$: 
		\begin{enumerate}
		    \item One of the scenarios 1, 2, or 3 happens. The algorithm will choose $\text{Committed}_{k+1}=\text{Committed}_{k}$, and by the assumption  \ref{assumption_theorem} we have that $\text{Committed}_{k+1} \subseteq\mathcal{F}_{k+1} \cup \mathcal{F}_{\text{Unknown},k+1}$
			\item In any other case, the trajectory obtained ($\text{Committed}_{k+1}$) will be inside $\mathcal{F}_{k+1}$ by construction of the algorithm. 
		\end{enumerate} 
		
		Hence, we conclude that
		\begin{eqnarray*}
	&&	\text{Committed}_k \subseteq\mathcal{F}_k \cup \mathcal{F}_{\text{Unknown},k}\\
	&&	\implies \text{Committed}_{k+1} \subseteq\mathcal{F}_{k+1} \cup \mathcal{F}_{\text{Unknown},k+1}
		\end{eqnarray*}
	\end{enumerate}
\end{proof}

\begin{remark}
    \label{remark2_theorem}
    The theorem does not assume that $\mathcal{F}_k \subseteq \mathcal{F}_{k+1}$. In other words, it does not assume that the size of the free-known space always increases: $\mathcal{F}_k \subseteq \mathcal{F}_{k+1}$ is not necessarily true due to the sliding map. Note, however, that the proof does not depend on the shape of the map nor on the length of the history kept in this map. Hence, the theorem is also valid for the following two cases:
    
    \begin{itemize}
        \item a nonsliding global map $\mathcal{M}\equiv\mathbb{R}^3$.
        \item a map $\mathcal{M}\equiv \text{FOV}$ (Field of View of the sensor), obtained uniquely by considering the instantaneous sensing data and, therefore, not keeping history in the map.
    \end{itemize}
    
\end{remark}

\begin{figure}[t]
	\centering
	\includegraphics[width=\columnwidth]{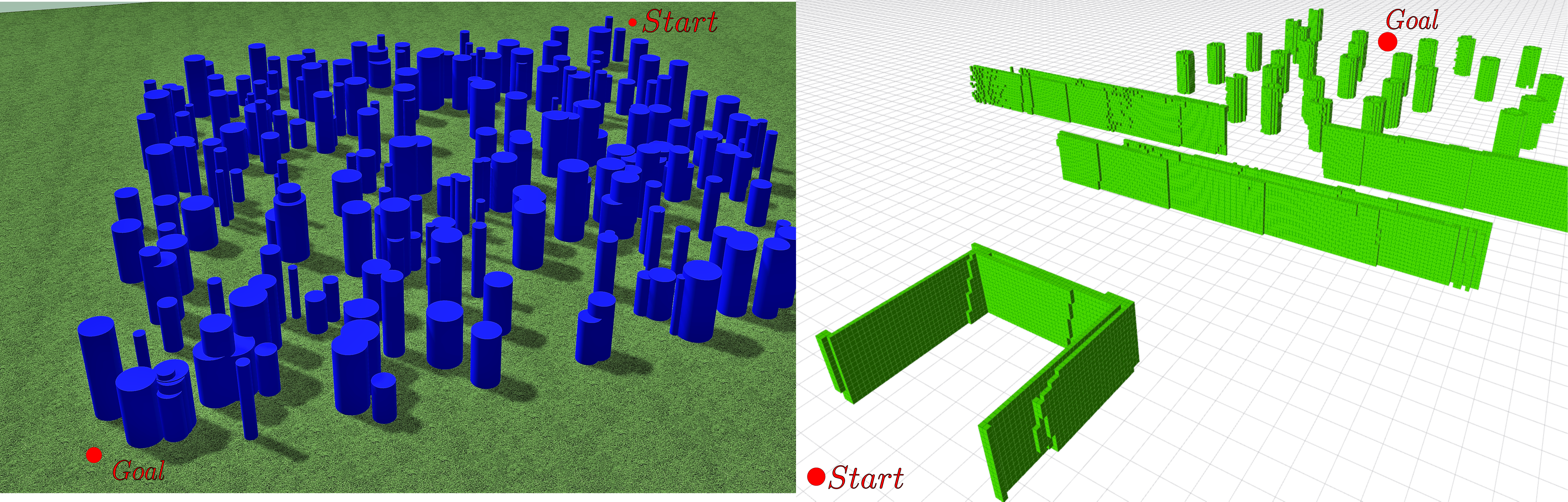}
	\caption[Forest and bugtrap environments used in the simulation]{Forest (left) and bugtrap (right) environments used in the simulation. The forest is $50$~m~$\times$~$50$~m, and the grid in the bugtrap environment is $1$~m~$\times$~$1$~m.}
	\label{fig:forest_and_bugtrap}
\end{figure}

\begin{remark}
    By allowing the algorithm to choose $\text{Committed}_{k+1}=\text{Committed}_{k}$ (which occurs when one of the scenarios 1, 2, or 3 happen), in iteration $k+1$ the UAV may commit to a trajectory that has some parts outside the map $\mathcal{M}_{k+1}$. As proven above, it is still guaranteed that $\text{Committed}_{k+1} \subseteq \mathcal{F}_{k+1} \cup \mathcal{F}_{\text{Unknown},k+1}$ . This constitutes a form of data compression, where the information of a part of the world being free (which was obtained in iteration ${k}$ or before) is embedded in the trajectory itself and not directly in the map $\mathcal{M}_{k+1}$.
\end{remark}

\begin{figure*}[]
	\centering
	\includegraphics[width=\textwidth]{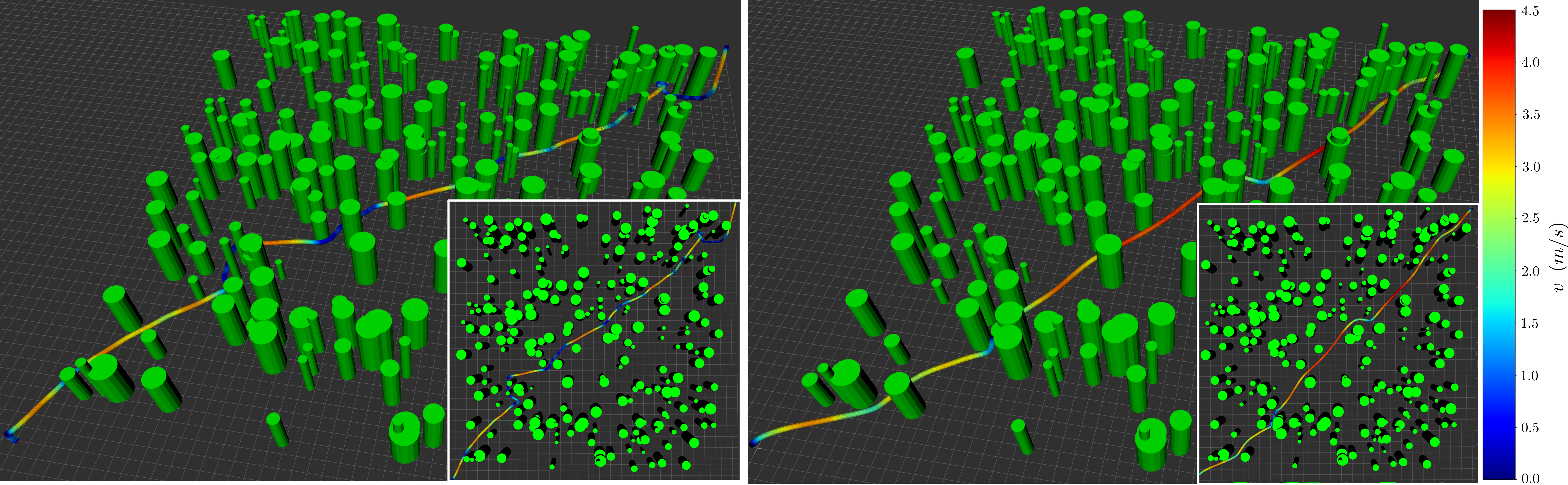}
	\caption[Velocity profile in a random forest simulation]{Velocity profile in a random forest simulation. On the left the results of our previous work \cite{tordesillas2018real} and on the right FASTER.}
	\label{fig:forest_comparison_vel}
\end{figure*}

\subsection{Controller}\label{subsec:controller_uav}
\add{To track the trajectory obtained by FASTER, we used the cascade controller presented in \cite{lopez2016low}. The yaw of the UAV is chosen such that the camera of the UAV points to $M$ (intersection between $\text{JPS}_k$ and $\mathcal{U}$, see Fig.~\ref{fig:plan2}). This controller is used in all the UAV simulation and hardware experiments of this article. In the real hardware experiments, position, velocity, attitude, and IMU biases are estimated by fusing propagated IMU measurements with an external motion capture system.}

\section{Simulation Results} \label{sec:sim_results}

\add{
\subsection{Forest, bugtrap and office simulations} \label{sec:forest_bugtrap_office_sim}
}

\begin{figure}[t]
\centering
\begin{minipage}{\columnwidth}
	\centering
	\includegraphics[width=\columnwidth,trim=0 0 0 50,clip]{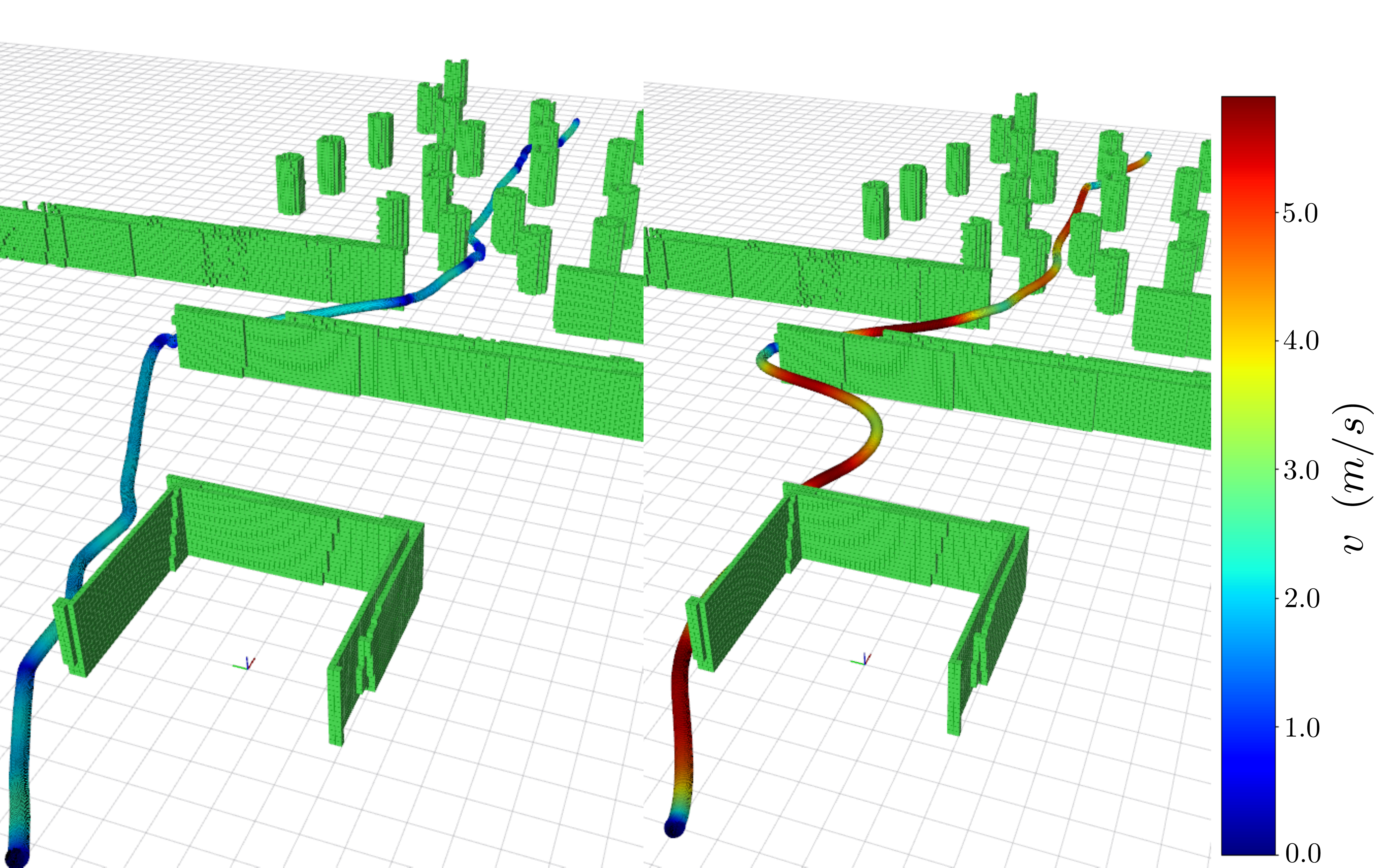}
	\captionof{figure}{Velocity profile in the bugtrap simulation. On the left the results of our previous work \cite{tordesillas2018real} and on the right FASTER.}
	\label{fig:bugtrap_comparison_vel}
\end{minipage}
\begin{minipage}{\columnwidth}

	\vspace*{.2in}
 	\renewcommand\thetable{III}%
	\captionof{table}{\label{tab:table_bugtrap} Comparison between \cite{tordesillas2018real} and FASTER of flight distances and times in a bugtrap simulation.}
	\centering 
	\begin{tabular}{p{2cm} >{\centering\arraybackslash}p{2cm} >{\raggedleft\arraybackslash}p{1.75cm} >{\raggedleft\arraybackslash}p{1.75cm}  }
		\hline
		\hline
		
		\multicolumn{1}{l}{\textbf{Method}}               & \textbf{Distance (m)}  & \textbf{Time (s)} \\ %
		Multi-Fid.         & 56.8                   & 37.6             \\ 
		\textbf{FASTER}                      & \textbf{55.2}         & \textbf{13.8}    \\ \hline
		\multicolumn{1}{l}{\hspace*{-.8em} \rule{0pt}{10pt} \textbf{Improvement (\%)}}    & \textbf{2.8}              & \textbf{63.3}                                    \\ \hline \hline
	\end{tabular}

\end{minipage}
\end{figure}

\begin{table}
\begin{center}
    \renewcommand\thetable{I}%
	\caption{\label{tab:table_forest_distance}Distances obtained in 10 random forest simulations. The distance values are computed for the cases that reach the goal. All the results (except the ones of \cite{tordesillas2018real} and FASTER) were provided by the authors of \cite{oleynikova2018safe}.}
	\begin{tabular}{p{1.8cm} >{\centering\arraybackslash}p{1.1cm} >{\raggedleft\arraybackslash}p{0.55cm} >{\raggedleft\arraybackslash}p{0.55cm} >{\raggedleft\arraybackslash}p{0.55cm} >{\raggedleft\arraybackslash}p{0.55cm}}
		\hline
		\hline
		\multicolumn{1}{l}{\textbf{Method}} & \multicolumn{1}{l}{\textbf{Number of}}     & \multicolumn{4}{c}{\textbf{Distance (m)}}                                                            \\ \cline{3-6} 
		\multicolumn{1}{l}{}               & \textbf{Successes} & \textbf{Avg}  & \textbf{Std} & \multicolumn{1}{l}{\textbf{Max}} & \multicolumn{1}{l}{\textbf{Min}} \\ %
		Incremental                        & 0                  & -             & -            & -                                & -                                \\ %
		Rand. Goals                        & \textbf{10}        & 138.0         & 32.0         & 210.5                            & 105.6                            \\ %
		Opt. RRT$^\star$                   & 9                  & 105.3         & 10.3         & 126.4                            & 95.5                             \\ %
		Cons. RRT$^\star$                  & 9                  & 155.8         & 52.6         & 267.9                            & 106.2                            \\ %
		NBVP    & 6                  & 159.3         & 45.6         & 246.9                            & 123.6                            \\ %
		SL Expl. & 8                  & 103.8         & 21.6         & 148.3                            & 86.6                             \\ 
		Multi-Fid.       &  \textbf{10}                 & 84.5          & 11.7         & 109.4                            & 73.2                             \\ 
		\textbf{FASTER}                      &  \textbf{10}                 & \textbf{77.6} & \textbf{5.9} & \textbf{88.0}                    & \textbf{70.7}                    \\ \hline
		\multicolumn{2}{l}{\hspace*{-.5em} \rule{0pt}{10pt} \textbf{Min/Max improv. (\%)}}                          & \textbf{8/51}              & \textbf{43/89} & \textbf{20/67} & \textbf{3/43}                                       \\ \hline \hline
	\end{tabular}
	\vspace*{.2in}			
	\renewcommand\thetable{II}%
	\caption{\label{tab:table_forest_time} Comparison between \cite{tordesillas2018real} and FASTER of flight times in the forest simulation. Results are for 10 random forests. }
	\begin{tabular}{p{1.2cm} >{\centering\arraybackslash}p{1.cm} >{\raggedleft\arraybackslash}p{0.75cm} >{\raggedleft\arraybackslash}p{0.75cm} >{\raggedleft\arraybackslash}p{0.75cm} >{\raggedleft\arraybackslash}p{0.75cm}}
		\hline
		\hline
		\multicolumn{1}{l}{\textbf{Method}} &       & \multicolumn{4}{c}{\textbf{Time (s)}}                                                            \\ \cline{3-6} 
		\multicolumn{1}{l}{}               &   & \textbf{Avg}  & \textbf{Std} & \multicolumn{1}{l}{\textbf{ Max}} & \multicolumn{1}{l}{\textbf {Min}} \\ %
		Multi-Fid.        &                   & 61.2          & 16.8         & 92.5                             & 37.9                             \\ 
		\textbf{FASTER}                      &          & \textbf{29.2} & \textbf{4.2} & \textbf{36.8}                    & \textbf{21.6}                    \\ \hline
		\multicolumn{2}{l}{\hspace*{-.5em} \rule{0pt}{10pt} \textbf{Improvement (\%)}}                          & \textbf{52.3}              & \textbf{75.0} & \textbf{60.2} & \textbf{43.0}                                       \\ \hline \hline
	\end{tabular}

\end{center}
\end{table}

\begin{figure}
\centering
\begin{minipage}{\columnwidth}
	\centering
	\includegraphics[width=1\columnwidth,trim=0 0 0 0,clip]{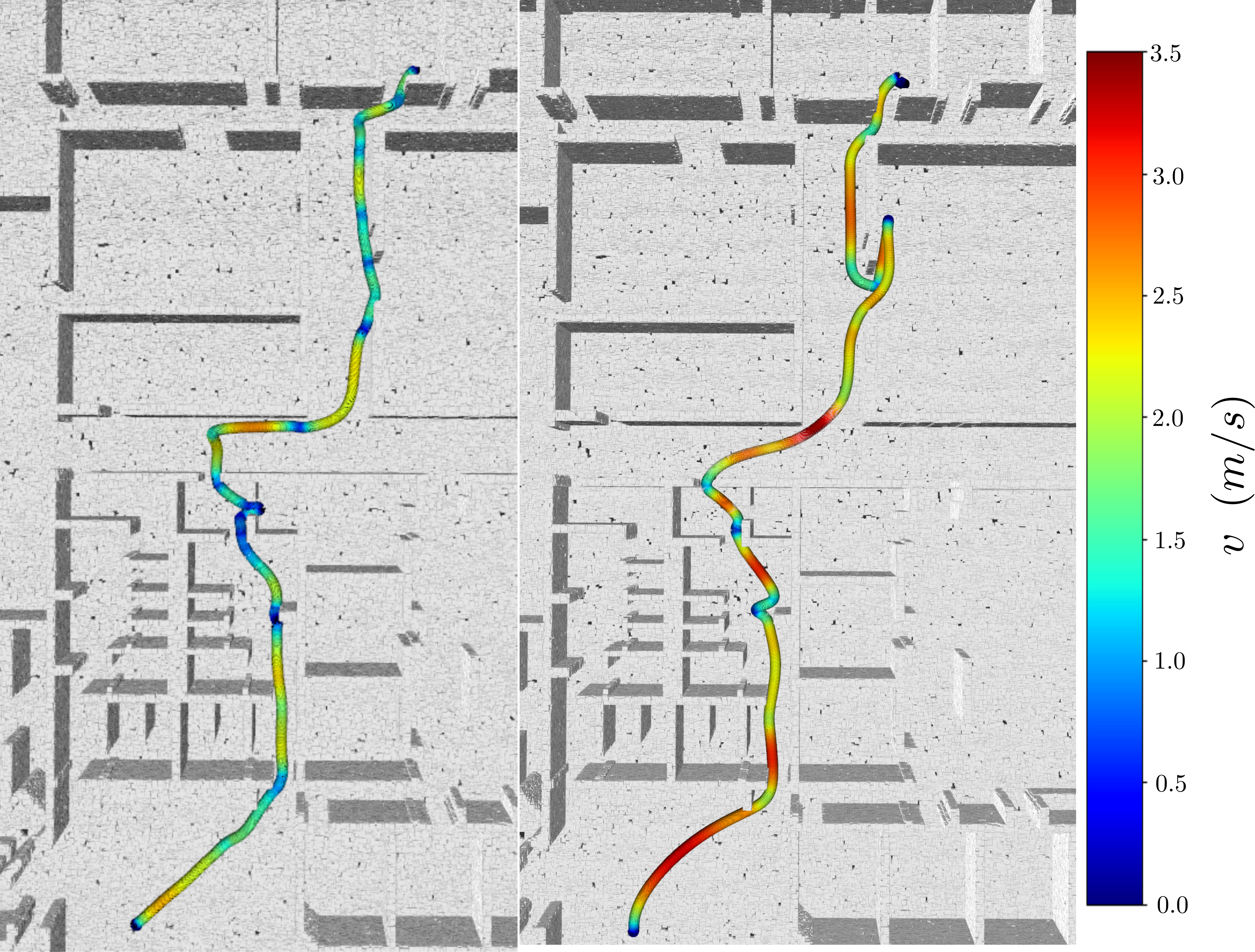}
	\captionof{figure}{Velocity profile in the office simulation. On the left the results of \cite{tordesillas2018real} and on the right FASTER.}
	\label{fig:office_comparison_vel}
\end{minipage}
\begin{minipage}{\columnwidth}
 	\vspace*{.2in}
 	\renewcommand\thetable{IV}%
	\captionof{table}{\label{tab:table_office} Comparison between \cite{tordesillas2018real} and FASTER of flight distances and times in an office simulation.}
	\centering 
	\begin{tabular}{p{2cm} >{\centering\arraybackslash}p{2cm} >{\raggedleft\arraybackslash}p{1.75cm} >{\raggedleft\arraybackslash}p{1.75cm}  }
		\hline
		\hline
		\multicolumn{1}{l}{\textbf{Method}}               & \textbf{Distance (m)}  & \textbf{Time (s)} \\ %
		Multi-Fid.         & \textbf{41.5}                   & 29.73             \\ 
		\textbf{FASTER}                      & 43.9         & \textbf{20.94}    \\ \hline
		\multicolumn{1}{l}{\hspace*{-.8em} \rule{0pt}{10pt} \textbf{Improvement (\%)}}    & \textbf{-5.8}              & \textbf{29.6}                                    \\ \hline \hline
	\end{tabular}
\end{minipage}
\end{figure}

We evaluate the performance of the proposed algorithm in different simulated scenarios. The simulator uses C++ custom code for the dynamics engine, integrating the nonlinear differential equations of the UAV using the Runge-Kutta method. Gazebo \cite{koenig2004design} is used to simulate perception data in the form of a depth map. In all these simulations, the depth camera has a horizontal FOV of $90^{\circ}$. The sensing range is $5$ m for the first simulation (corner environment), and $10$ m for the rest. 

We now test FASTER in 10 random forest environments with an obstacle density of $0.1$~obstacles/m$^2$ (see Fig.~\ref{fig:forest_and_bugtrap}), and compare the flight distances achieved against the following seven approaches:
\begin{itemize}
	\item Incremental approach (no goal selection).
	\item Random goal selection.
    \item Optimistic RRT$^\star$ (unknown space = free).
    \item Conservative RRT$^\star$ (unknown space = occupied).
	\item ``Next-best-view'' planner (NBVP) \cite{bircher2016receding}.
	\item  Safe Local Exploration \cite{oleynikova2018safe}.
	\item  Multi-Fidelity \cite{tordesillas2018real}.
\end{itemize}

The first six methods are \add{described in \cite{oleynikova2018safe} and \cite{tordesillas2018real} is our previous algorithm}. The results in Table~\ref{tab:table_forest_distance} highlight that FASTER achieves a $8-51\%$ improvement in the total distance flown. Completion times are compared in Table \ref{tab:table_forest_time} to \cite{tordesillas2018real} (time values are not available for all other algorithms in Table~\ref{tab:table_forest_distance}). FASTER achieves an improvement of $52\%$ in the completion time. The dynamic constraints imposed for the results of this table are (per axis) $v_{\text{max}}=5$~m/s, $a_{\text{max}}= 5$~m/s$^2$, and $j_{\text{max}}= 8$~m/s$^3$. The velocity profiles obtained for one random forest simulation are shown in Fig.~\ref{fig:forest_comparison_vel}.

We also test FASTER using the bugtrap environment shown in Fig.~\ref{fig:forest_and_bugtrap}, and obtain the results that appear on Table \ref{tab:table_bugtrap}. Both algorithms have a similar total distance, but FASTER achieves an improvement of $63\%$ on the total flight time. For both cases, the dynamic constraints imposed are $v_{\text{max}}=10$~m/s, $a_{\text{max}}= 10$~m/s$^2$, and $j_{\text{max}}= 40$~m/s$^3$. The velocity profile achieved along the trajectory can be seen in Fig.~\ref{fig:bugtrap_comparison_vel}. 

Finally, we test FASTER in an office environment, obtaining the velocity profile shown in Fig \ref{fig:office_comparison_vel} and the distances and flight times shown in Table \ref{tab:table_office}. In this case, the distance flown by FASTER was slightly longer than the one by \cite{tordesillas2018real} (note that FASTER entered one of the last rooms, and then turned back), but even with this extra distance, it achieved a $29.6\%$ improvement on the flight time. The dynamic constraints used for the office simulation are $v_{\text{max}}=3$~m/s, $a_{\text{max}}= 6$~m/s$^2$ and $j_{\text{max}}= 35$~m/s$^3$.

\begin{figure}[t]
	\centering
	\includegraphics[width=1\columnwidth]{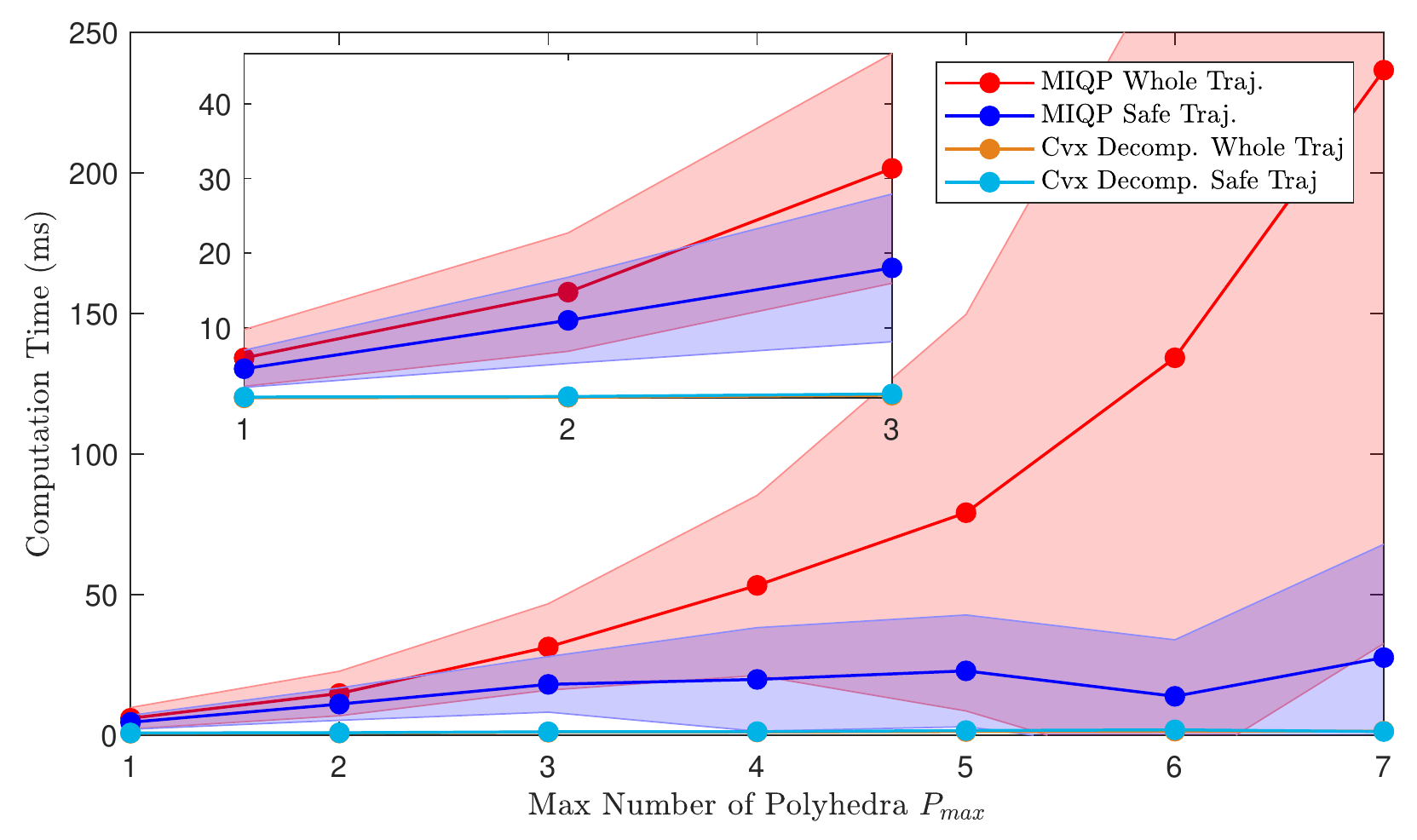}
	\caption{Timing breakdown for the MIQP and Convex Decomposition of the Whole Trajectory and the Safe Trajectory as a function of the maximum number of polyhedra $P_{\text{max}}$ \add{for the forest simulation}. Note that the times for the MIQPs include all the trials until convergence (with different factors $f$) in each replanning step. The shaded area is the 1-$\sigma$ interval ($\sigma$ is the standard deviation).}
	\label{fig:timing_all}
	\vspace*{.1in}
	\centering
	\includegraphics[width=\columnwidth]{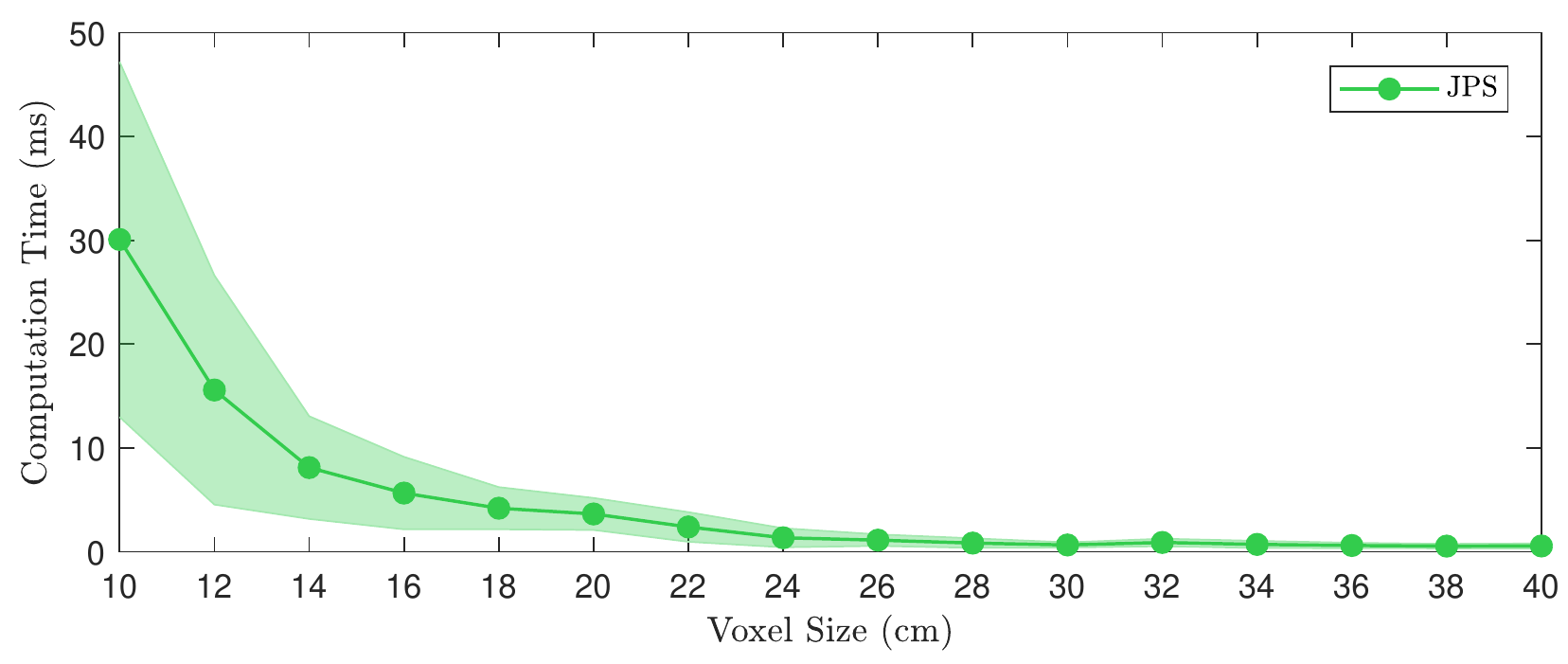}
	\caption{Runtimes of JPS as a function of the voxel size. The shaded area is the 1-$\sigma$ interval \add{($\sigma$ is the standard deviation). These results are for the forest simulation using} a sliding map of size $20$~m~$\times$~$20$~m. }
	\label{fig:timing_jps}

\end{figure} 

\begin{figure*}[]
	\centering
	\newcommand{\TAs}{\textbf{TAs}}
	\newcommand{\TA}{\textbf{TA}}
	\newcommand{\IA}{\textbf{IA}}
	\includegraphics[width=\textwidth]{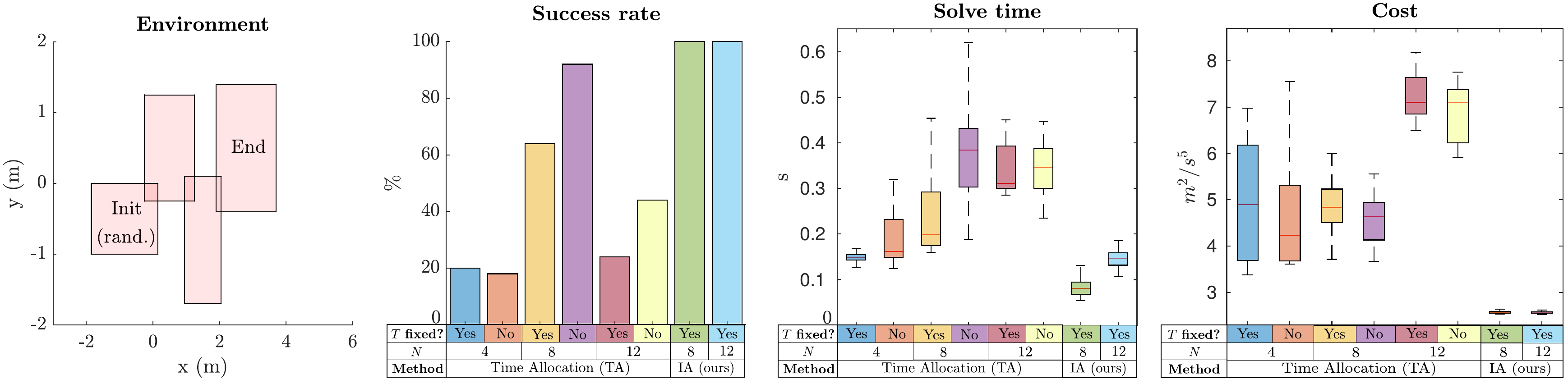}
\caption{Time (\TA{}) vs. Interval (\IA{}) allocation for different number intervals $N$ and different constraints on the total time of the trajectory $T$ (free vs. fixed). In all the \TA{} methods, there are $N/4$ intervals per polyhedron, where $N$ is the total number of intervals.
\IA{} has a fixed time allocation, and uses binary variables to optimize the allocation of the $N$ intervals. The plot on the left shows the 2-D projection of the 3-D flight corridor used in the experiments. The initial position is chosen randomly in the first polyhedron, and the end position is fixed inside the fourth polyhedron. The total cost in these experiments is computed as \costplots{}. For every method, a total of 50 runs are performed, and only the successful runs were taken into account for the costs and solve times. }
\label{fig:allocations}
\end{figure*} 

\begin{figure*}[]
	\centering
	\includegraphics[width=\textwidth]{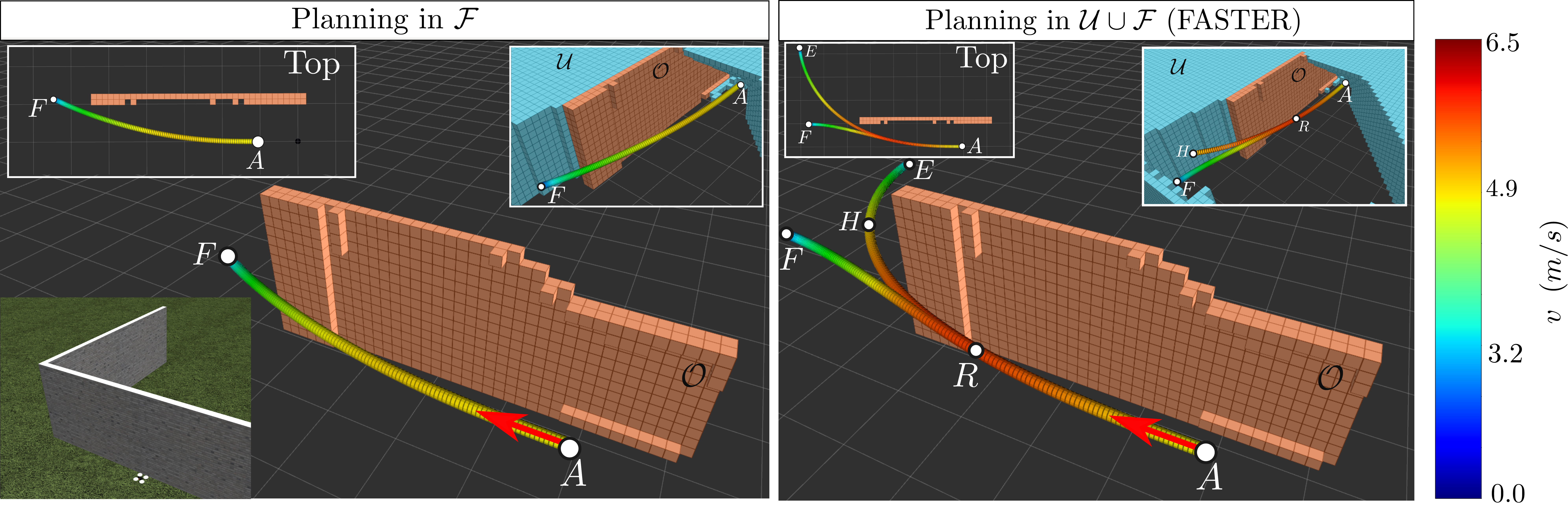}
	\caption{Trajectories obtained when planning only in $\mathcal{F}$ (left) and when planning in $\mathcal{F} \cup \mathcal{U}$ (FASTER, right). The velocity at $A$ is $4.8$ m/s. FASTER achieves a velocity of $6.02$ m/s in the segment $A\rightarrow R$ (segment that will actually be flown by the UAV), while the other planner achieves a velocity of $5.06$ m/s. The ground grid is $1$~m~$\times$~$1$~m.}
	\label{fig:comparison_faster_onlyfree}
\end{figure*} 

\definecolor{DarkGreen}{rgb}{0,0.5,0}
\newcommand{\Yesc}{\textcolor{DarkGreen}{\textbf{Yes}}}
\newcommand{\Noc}{\textcolor{red}{\textbf{No}}}

\begin{figure}
  \centering
  \includegraphics[width=1\columnwidth]{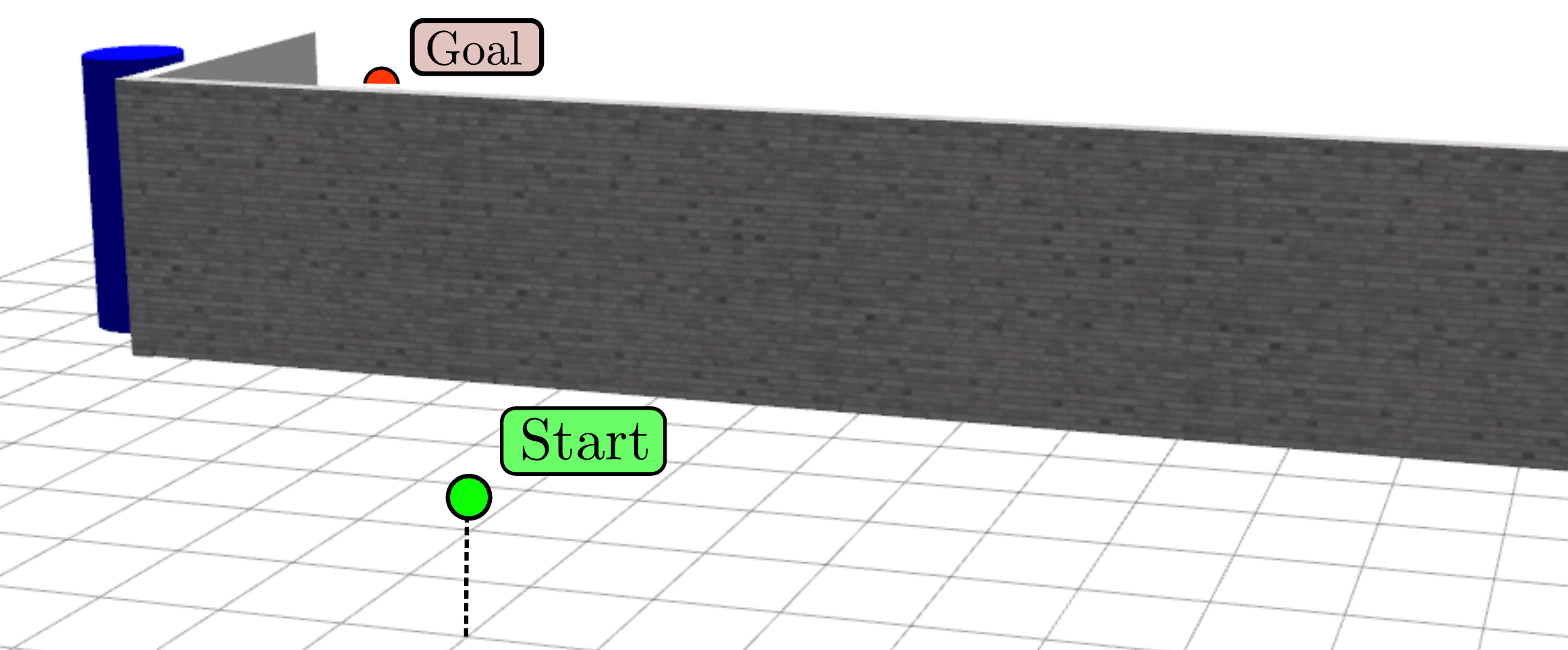}
  \caption{\add{Environment with an obstacle behind the corner.}} \label{fig:environment_cylinder_corner}
  \captionof{table}{\add{Safety with and without the Safe Trajectory. The ratios represent the successful runs (i.e., without crashes), for a total of 5 runs.}}\label{tab:with_without_safe}
    \add{\begin{tabular}{cccc}
    \hline 
    ${v_{\text{max}}}$\qquad& \textbf{\qquad FASTER\qquad} & \textbf{ No safe traj.\qquad}\tabularnewline
    \hline
    \hline 
    \textbf{$\boldsymbol{4}$  m/s}& \Yesc{} (5/5) & \Yesc{} (5/5) \tabularnewline
    \textbf{$\boldsymbol{6}$  m/s}& \Yesc{} (5/5) & \Noc{} (2/5) \tabularnewline
    \textbf{$\boldsymbol{8}$  m/s}& \Yesc{} (5/5) & \Noc{} (0/5) \tabularnewline
    \hline
    \vspace{1mm}
\end{tabular}
}
\end{figure}

The timing breakdown of Alg.~\ref{algo: myalgorithm_iros} as a function of the maximum number of polyhedra $P_{\text{max}}$ is shown in Fig.~\ref{fig:timing_all}. The number of intervals $N$ was 10 for the Whole Trajectory and 7 for the Safe Trajectory. Note that the runtime for the MIQP of the Safe Trajectory is approximately constant as a function of $P_{\text{max}}$ \add{because the Safe Trajectory is planned only in $\mathcal{F}$, and therefore, most of the time, $P < P_{\text{max}}$. For the simulation and hardware experiments presented here, $P_{\text{max}} = 2-4$ was used. Fig.~\ref{fig:timing_jps} shows the runtimes for JPS as a function of the voxel size of the map, which are always $<10$ ms for voxel sizes $\ge 14$ cm}. All these timing breakdowns were measured using an Intel Core i7-7700HQ 2.8GHz Processor.

\subsection{Time vs. Interval allocation}\label{subsec:timeVsIntervalAlloc}
\newcommand{\TA}{\textbf{TA}}
\newcommand{\IA}{\textbf{IA}}
As explained in Sec.~\ref{sec:algorithm}, FASTER optimizes the interval allocation using binary variables, while fixing in each optimization the time allocated per interval. Another possible option would be to optimize the time allocation, while fixing the interval allocation. To see the advantages and disadvantages of each option, we compare the following two approaches: 
\begin{itemize}  
 \item \TA: \textbf{T}ime \textbf{A}llocation is optimized and there are $N/P$ intervals per polyhedron. We test both the case when the total time of the trajectory $T$ is free and when it is fixed at $12.5$ s.
 \item \IA \textbf{ (ours)}: \textbf{I}nterval \textbf{A}llocation is optimized and all the intervals have the same fixed allocated time. \IA{} uses binary variables to optimize the allocation of the $N$ intervals. $T$ is fixed at $12.5$ s and the time allocated per interval is $12.5/N$.
\end{itemize}
We use an environment 
whose free space is defined by 4 overlapping polyhedra (i.e., $P=4$, see Fig.~\ref{fig:allocations}). The final state is a stop condition in the centroid of the last polyhedron, while the initial state is a stop condition in a random position of the first polyhedron, for a total of 50 runs. 
Both \IA{} and \TA{} methods use a weighted sum of the control effort and the total time as the total cost: \costplots{}. Note that the second term of this cost is constant for the methods in which $T$ is not a decision variable. The dynamic constraints imposed are $v_{\text{max}}=2$~m/s, $a_{\text{max}}= 20$~m/s$^2$, and $j_{\text{max}}= 50$~m/s$^3$. The solver used for the (nonconvex) problems of \TA{} is \emph{fmincon} \cite{matlabOptToolbox}, while \emph{Gurobi} \cite{gurobi} is used for the MIQP of \IA{} (both interfaced through \emph{YALMIP} \cite{Lofberg2004,Lofberg2009}).
The results in Fig.~\ref{fig:allocations} show that \IA{} is able to succeed in all of the runs, and it obtains smaller total costs and computation times.
The \TA{} methods achieve lower success rates, though these tend to increase when $T$ is not fixed and $N>P$.  
All these results support the choice of optimizing the interval allocation (instead of the time allocation) that FASTER makes. Note also that, as explained in Sec.~\ref{subsec:local_planner}, FASTER runs on top of this a line search to choose the time allocated per interval, see Fig.~\ref{fig:dynamic_adaptation_factor}.

\add{
  \subsection{Role of the Safe Trajectory}\label{sec:roleSafeTraj}
}
\subsubsection{\add{Speed achieved}}
We first test FASTER in a simple environment and, for the same replanning step, we compare the velocities of the trajectory found by FASTER (that plans in $\mathcal{U}\cup \mathcal{F}$) with the ones of the trajectory found by a planner that plans only in $\mathcal{F}$. The environment is shown in Fig.~\ref{fig:comparison_faster_onlyfree}, and consists of a corner, with  the goal on the other side of the wall, so that the UAV has to turn the corner. The initial velocity at $A$ is $4.8$ m/s, and the dynamic constraints imposed are $v_{\text{max}}=6.5$~m/s, $a_{\text{max}}= 6$~m/s$^2$, and $j_{\text{max}}= 20$~m/s$^3$. FASTER achieves a velocity of $6.02$ m/s in the segment $A\rightarrow R$ (segment that will actually be flown by the UAV), while planning only in $\mathcal{F}$ achieves a velocity of $5.06$ m/s. $R\rightarrow F$ is the Safe Trajectory, and $A\rightarrow R \rightarrow F$ is the Committed Trajectory. Safety is guaranteed by both planners. 

\subsubsection{\add{Safety}}
\add{We now evaluate what happens if the UAV does not compute the Safe Trajectory, but instead commits directly to the Whole Trajectory. We test this in the environment shown in Fig.~\ref{fig:environment_cylinder_corner}, which consists of a corner with one obstacle behind it. This environment is especially challenging due to the presence of obstacles just behind the corner, which are not fully visible to the UAV until it turns the corner. The results in Table \ref{tab:with_without_safe}  show that the Safe Trajectory is not strictly necessary when flying at low speeds ($\leq 4$ m/s), but it is crucial to guarantee safety when flying at high speeds ($\ge 6$ m/s). For high speeds, the planner without the Safe Trajectory collides due to the lack of time to replan when suddenly discovering an obstacle that was in the unknown space.
 }

\add{\subsection{Comparison between $\text{Poly}_{\text{Whole}}$ and $\text{Poly}_{\text{Safe}}$}}\label{subsec:comparisonPolys}

\add{For the corner environment explained in Sec. \ref{sec:roleSafeTraj} (which uses 4 polyhedra), the top view and the quantitative comparison of the volumes covered are shown in Fig.~\ref{fig:comparison_volume_corner}. $\text{Poly}_{\text{Whole}}$ covers $145.1\cdot V_{\text{UAV}}$
of unknown space that extends beyond $\text{Poly}_{\text{Safe}}$. Here, $ V_{\text{UAV}}$ is the volume of the drone (a sphere of radius 0.3 m).}

\add{For the forest and office simulations (which use 2 polyhedra), the comparison of the volumes is shown in Fig.~\ref{fig:comparison_volume_forest_office} and Table \ref{tab:comparison_volume_forest_office}. 	
Letting $V_{\text{UAV}}$ denote the volume of the sphere that models the UAV, these results show that, on average, $\text{Poly}_{\text{Whole}}$ is, respectively, $250.8\cdot V_{\text{UAV}}$ and $21.9\cdot V_{\text{UAV}}$ larger than  $\text{Poly}_{\text{Safe}}$ in the office and forest simulations.  Moreover, $\text{Poly}_{\text{Safe}}$ does not cover unknown space, while $\text{Poly}_{\text{Whole}}$ is able to cover, respectively, an unknown volume of $122.8\cdot V_{\text{UAV}}$ and $5.5\cdot V_{\text{UAV}}$ in the office and forest simulations. Note also that in the office simulation (which is more cluttered than the forest simulation), $\text{Poly}_{\text{Whole}}$ covers more unknown volume than in the forest simulation.}

\add{The key conclusion of these results is that, even with a relatively small number of polyhedra (2-4), the volume of unknown space covered by $\text{Poly}_{\text{Whole}}$ can be hundreds of times the volume of the UAV, especially in cluttered environments. This makes $\text{Poly}_{\text{Whole}}$ extend much farther than $\text{Poly}_{\text{Safe}}$, which is restricted to stay in $\mathcal{F}$. Hence, the Whole Trajectory will benefit from a longer planning horizon, leading to a higher nominal speed in the segment $A\rightarrow R$ of the Whole Trajectory used in the Committed Trajectory.}

\begin{figure}
	\centering
	
	\includegraphics[width=1\columnwidth]{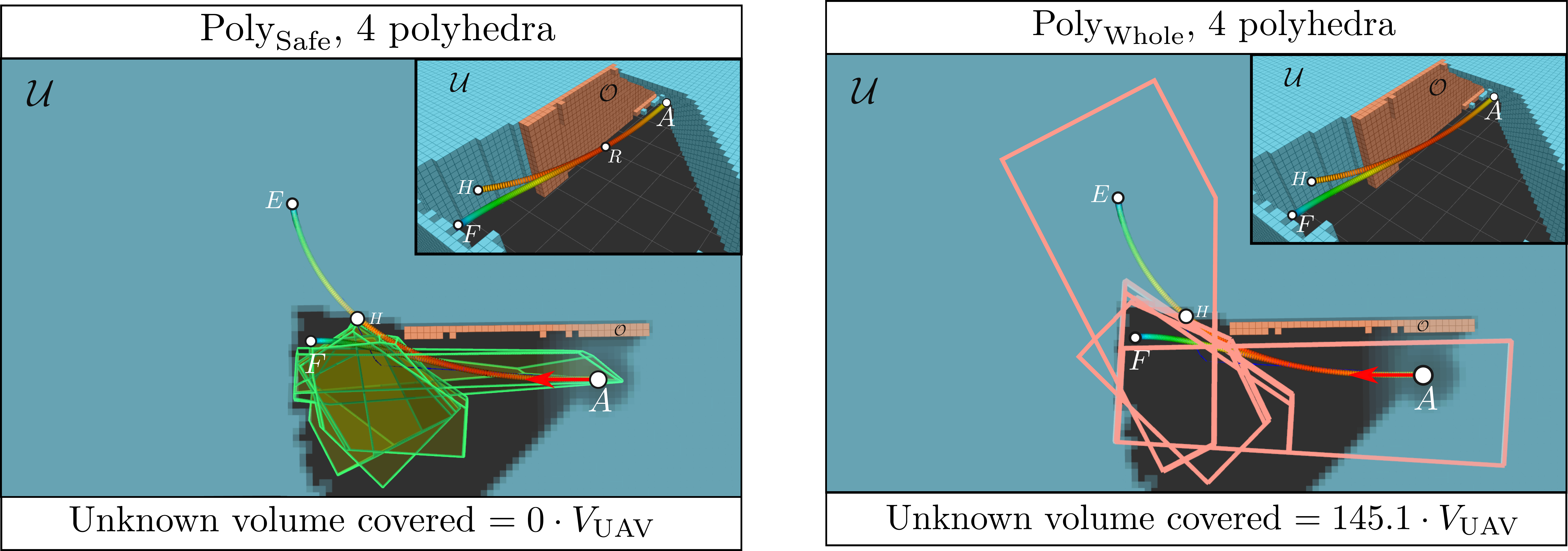}
	\caption{\add{Comparison of the unknown volume covered by $\text{Poly}_{\text{Safe}}$ and $\text{Poly}_{\text{Whole}}$ in the corner environment. As $\text{Poly}_{\text{Safe}}\subset\mathcal{F}$, it does not cover any unknown volume. However, $\text{Poly}_{\text{Whole}}\subset \mathcal{U}\cup\mathcal{F}$, and the total unknown volume covered is $145.1\cdot V_{\text{UAV}}$, where $V_{\text{UAV}}$ is the volume of a sphere with radius $r=0.3$ m that models the UAV. This makes optimization \ref{eq:MIQP} operate in a completely different space when using  $\text{Poly}_{\text{Safe}}$ than when using $\text{Poly}_{\text{Whole}}$.}}
	\label{fig:comparison_volume_corner}
\end{figure} 

\definecolor{wholeColor}{RGB}{255,198,191}
\definecolor{safeColor}{RGB}{190,254,207}
\begin{figure}
	\centering
	\includegraphics[width=1\columnwidth]{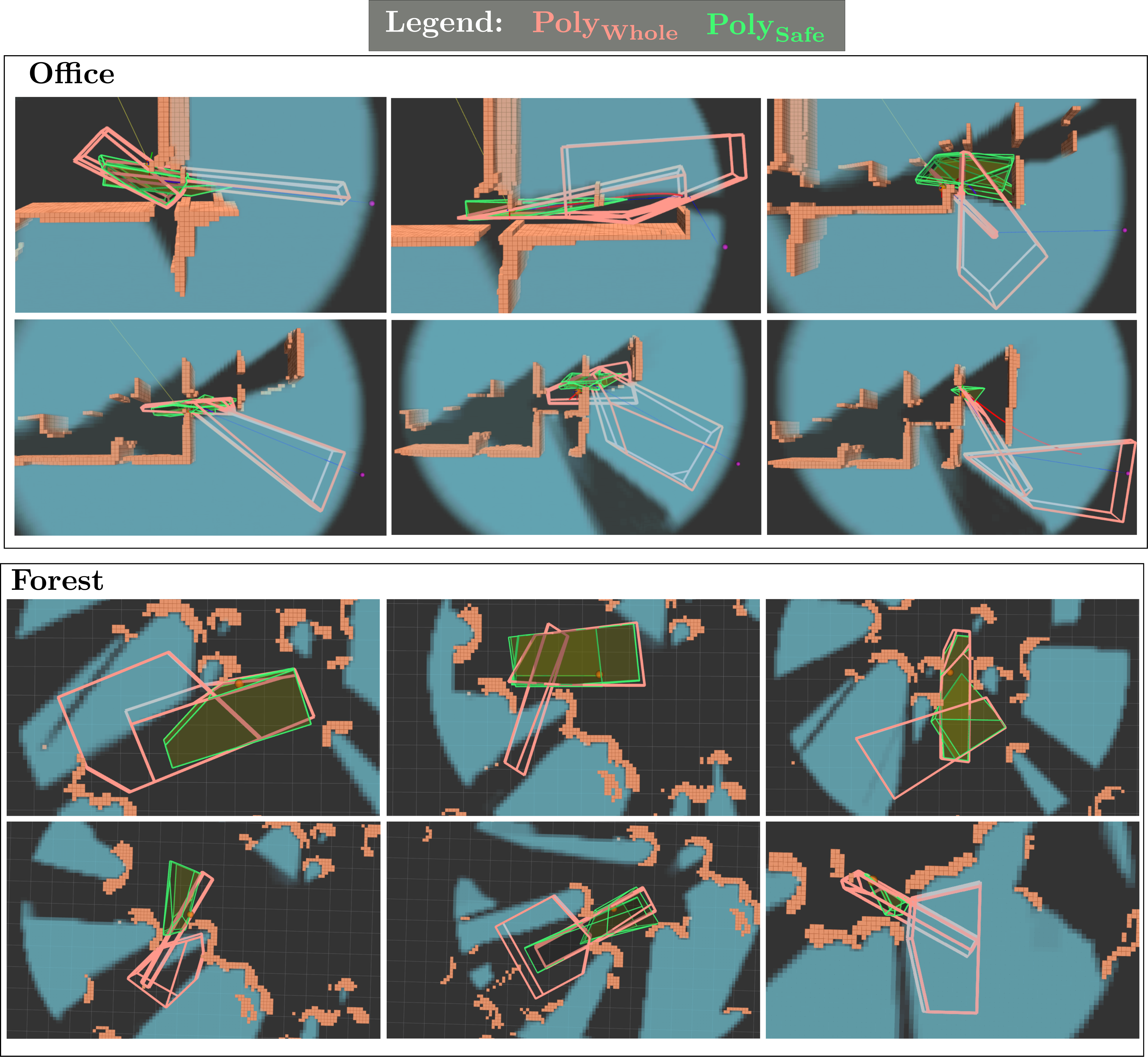}
\caption{\add{Qualitative comparison of the volumes covered by $\text{Poly}_{\text{Whole}}$ and $\text{Poly}_{\text{Safe}}$ in the office and forest simulations.}}
	\label{fig:comparison_volume_forest_office}
	
	\captionof{table}{\add{Quantitative comparison of the volumes covered by $\text{Poly}_{\text{Whole}}$ and $\text{Poly}_{\text{Safe}}$ in the forest and office  simulations. $V_{\text{UAV}}$ denotes the volume of the UAV, which is modeled as a sphere.}}\label{tab:comparison_volume_forest_office}
	\noindent\resizebox{\columnwidth}{!}{%
		\add{\begin{tabular}{|c|c|c|}
		\hline 
		& \textbf{Office simulation} & \textbf{Forest simulation}\tabularnewline
		\hline 
		\hline 
		\textbf{UAV model} & Sphere of $r=0.2$0 m & Sphere of $r=0.42$ m\tabularnewline
		\hline 
		\hline \cellcolor{safeColor}$\boldsymbol{\text{\textbf{vol}}\left(\mathrm{Poly}_{\mathrm{Safe}}\right)}$ & \cellcolor{safeColor}$\left(390.1\pm341.8\right)V_{\text{UAV}}$ & \cellcolor{safeColor}$\left(68.0\pm36.1\right)V_{\text{UAV}}$\tabularnewline
		\hline 
		\cellcolor{wholeColor}$\boldsymbol{\text{\textbf{vol}}\left(\mathrm{Poly}_{\mathrm{Whole}}\right)}$ & \cellcolor{wholeColor}$\left(\boldsymbol{640.9}\pm442.4\right)V_{\text{UAV}}$ & \cellcolor{wholeColor}$\left(\boldsymbol{89.9}\pm38.8\right)V_{\text{UAV}}$\tabularnewline
		\hline 
		\hline \cellcolor{safeColor}$\boldsymbol{\text{\textbf{vol}}\left(\mathrm{Poly}_{\mathrm{Safe}}\cap\mathcal{U}\right)}$ & \cellcolor{safeColor}$0.0\;V_{\text{UAV}}$ & \cellcolor{safeColor}$0.0\;V_{\text{UAV}}$\tabularnewline
		\hline 
		\cellcolor{wholeColor}$\boldsymbol{\text{\textbf{vol}}\left(\mathrm{Poly}_{\mathrm{Whole}}\cap\mathcal{U}\right)}$ & \cellcolor{wholeColor}$\left(\boldsymbol{122.8}\pm184.7\right)V_{\text{UAV}}$ & \cellcolor{wholeColor}$\left(\boldsymbol{5.5}\pm9.8\right)V_{\text{UAV}}$\tabularnewline
		\hline 
\end{tabular}}}
\end{figure} 

\begin{figure}
	\centering
	\includegraphics[width=0.8\columnwidth,trim=20 20 0 20,clip]{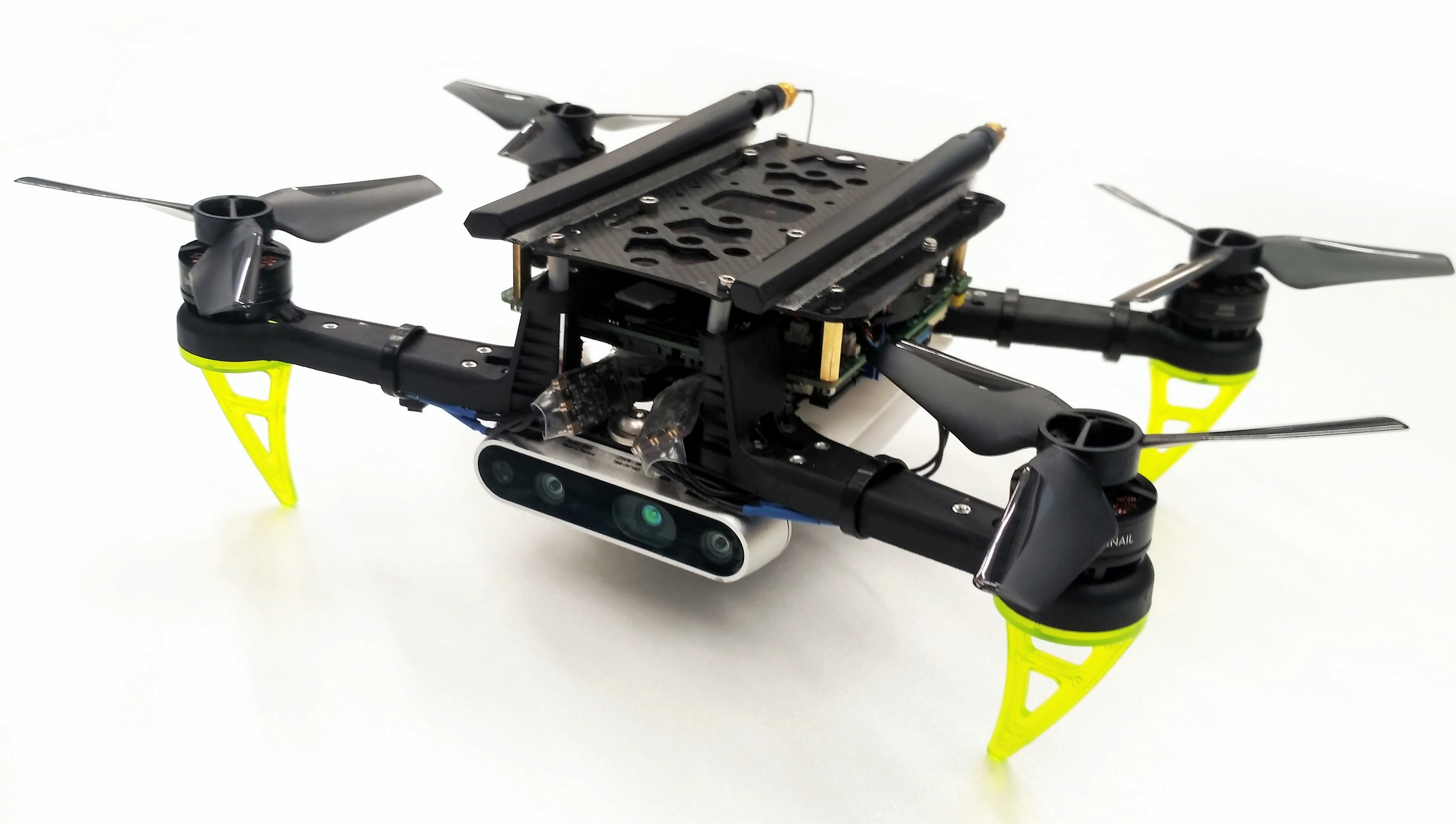}
	\includegraphics[width=0.8\columnwidth,trim=20 20 0 20,clip]{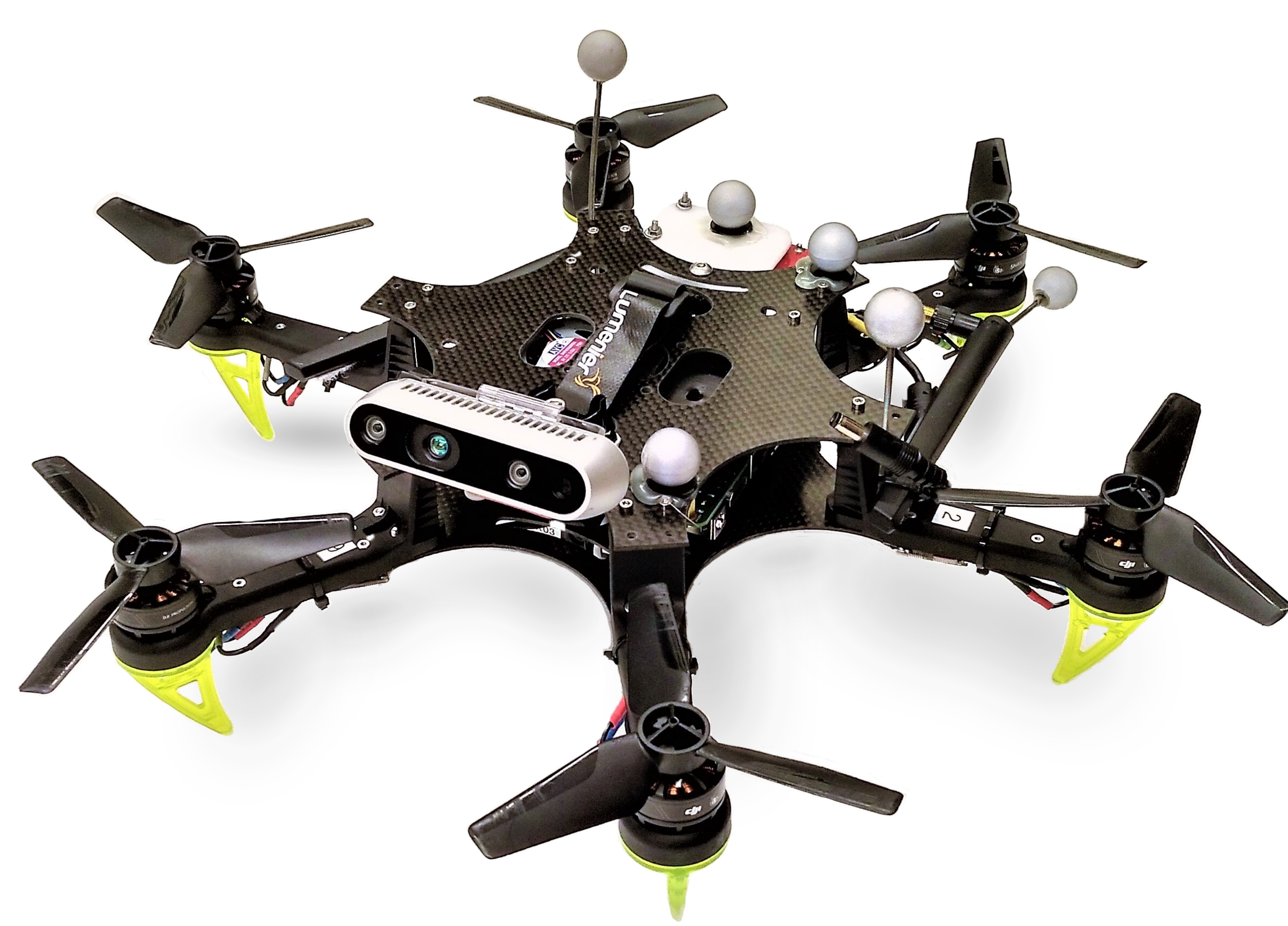}
	\caption{Quadrotor (top) used in the experiments 1-4 and hexarotor (bottom) used in the experiments 5 and 6. Both are equipped with a Qualcomm\textsuperscript{\textregistered} SnapDragon Flight, an Intel\textsuperscript{\textregistered} NUC i7DNK, and an Intel\textsuperscript{\textregistered} RealSense Depth Camera D435.  }
	\label{fig:drone}
\end{figure}

\begin{figure*}[]
	\includegraphics[width=\textwidth]{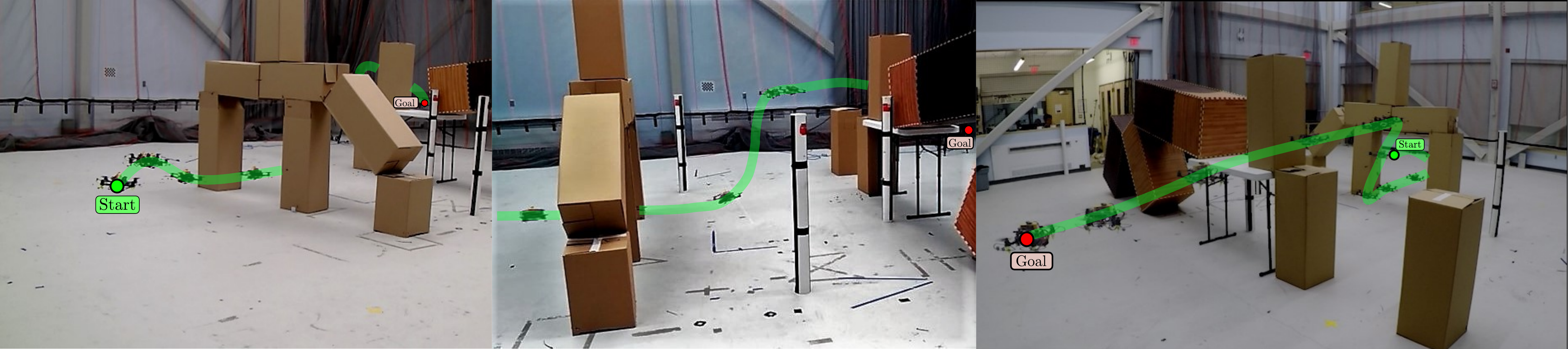}
	\caption[Composite images of Experiment~1]{Composite images of Experiment~1. The UAV must fly from start \tikzcircle[black,fill=green]{2pt} to goal \tikzcircle[black,fill=red]{2pt}. Snapshots shown every 670~ms.}
	\label{fig:exp1}
	\centering
	\includegraphics[width=\textwidth]{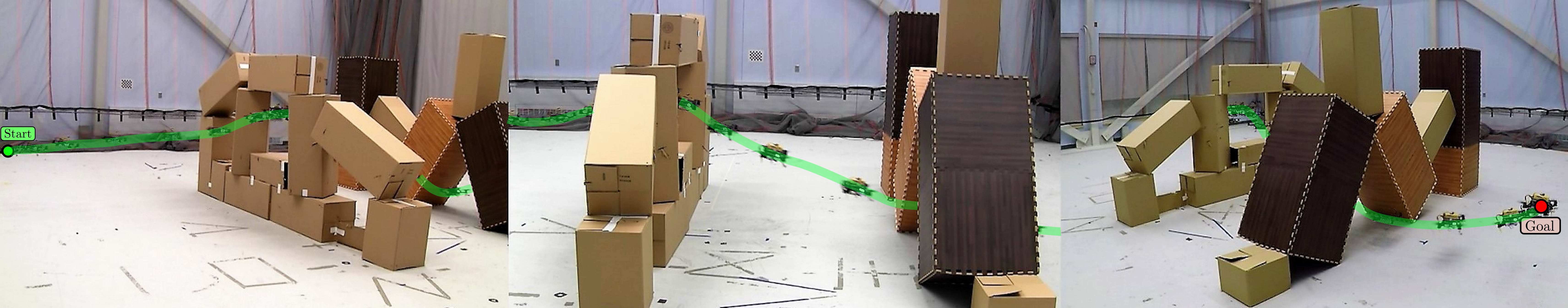}
	\caption[Composite images of Experiment~2]{ Composite image of Experiment 2. The UAV must fly from start \tikzcircle[black,fill=green]{2pt} to goal \tikzcircle[black,fill=red]{2pt}. Snapshots shown every 330~ms.}
	\label{fig:exp2}
	\centering
	\includegraphics[width=\textwidth]{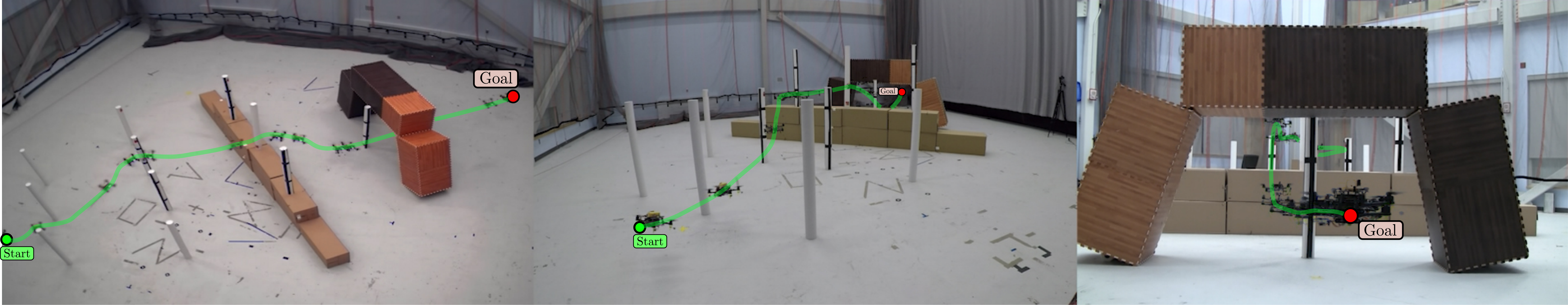}
	\caption[Composite images of Experiment~3]{ Composite image of Experiment 3. The UAV must fly from start \tikzcircle[black,fill=green]{2pt} to goal \tikzcircle[black,fill=red]{2pt}. Snapshots shown every 670~ms.}
	\label{fig:exp3}
	\centering
	\includegraphics[width=\textwidth]{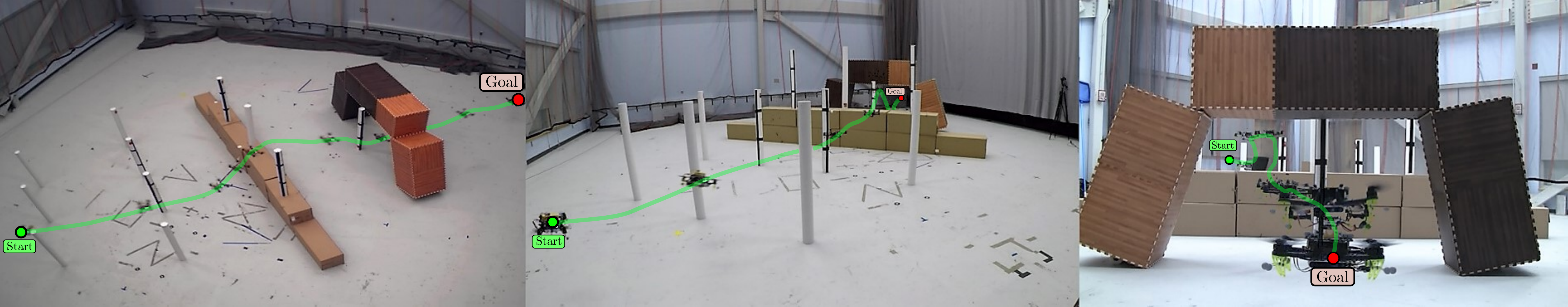}
	\caption[Composite images of Experiment~4]{ Composite image of Experiment 4. The UAV must fly from start \tikzcircle[black,fill=green]{2pt} to goal \tikzcircle[black,fill=red]{2pt}. Snapshots shown every 670~ms.}
	\label{fig:exp4}
	
	\centering
	\includegraphics[width=\textwidth]{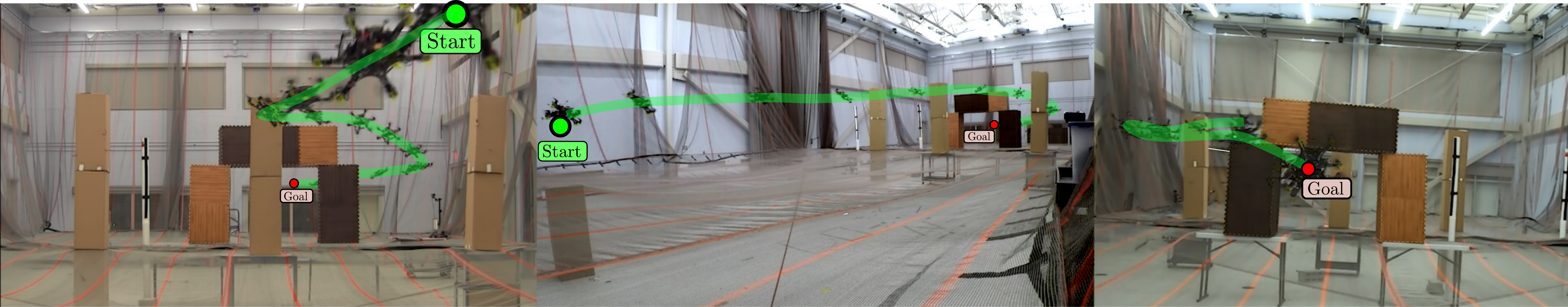}
	\caption[Composite images of Experiment~5]{ Composite image of Experiment 5. The UAV must fly from start \tikzcircle[black,fill=green]{2pt} to goal \tikzcircle[black,fill=red]{2pt}. Snapshots shown every 330~ms.}
	\label{fig:exp5}
	\centering
	\includegraphics[width=\textwidth]{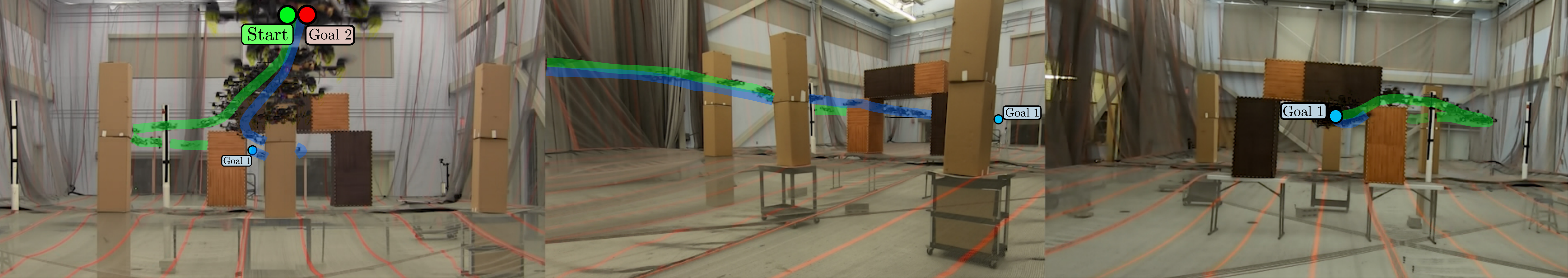}
	\caption[Composite images of Experiment~6]{ Composite image of Experiment 6. The UAV must fly from start \tikzcircle[black,fill=green]{2pt} to goal 1 \tikzcircle[black,fill=LightBlue]{2pt} and then back to goal 2 \tikzcircle[black,fill=red]{2pt}. Snapshots shown every 330~ms.}
	\label{fig:exp6}
	
\end{figure*} 

\begin{figure*}[t]
	\centering
	\includegraphics[width=\textwidth]{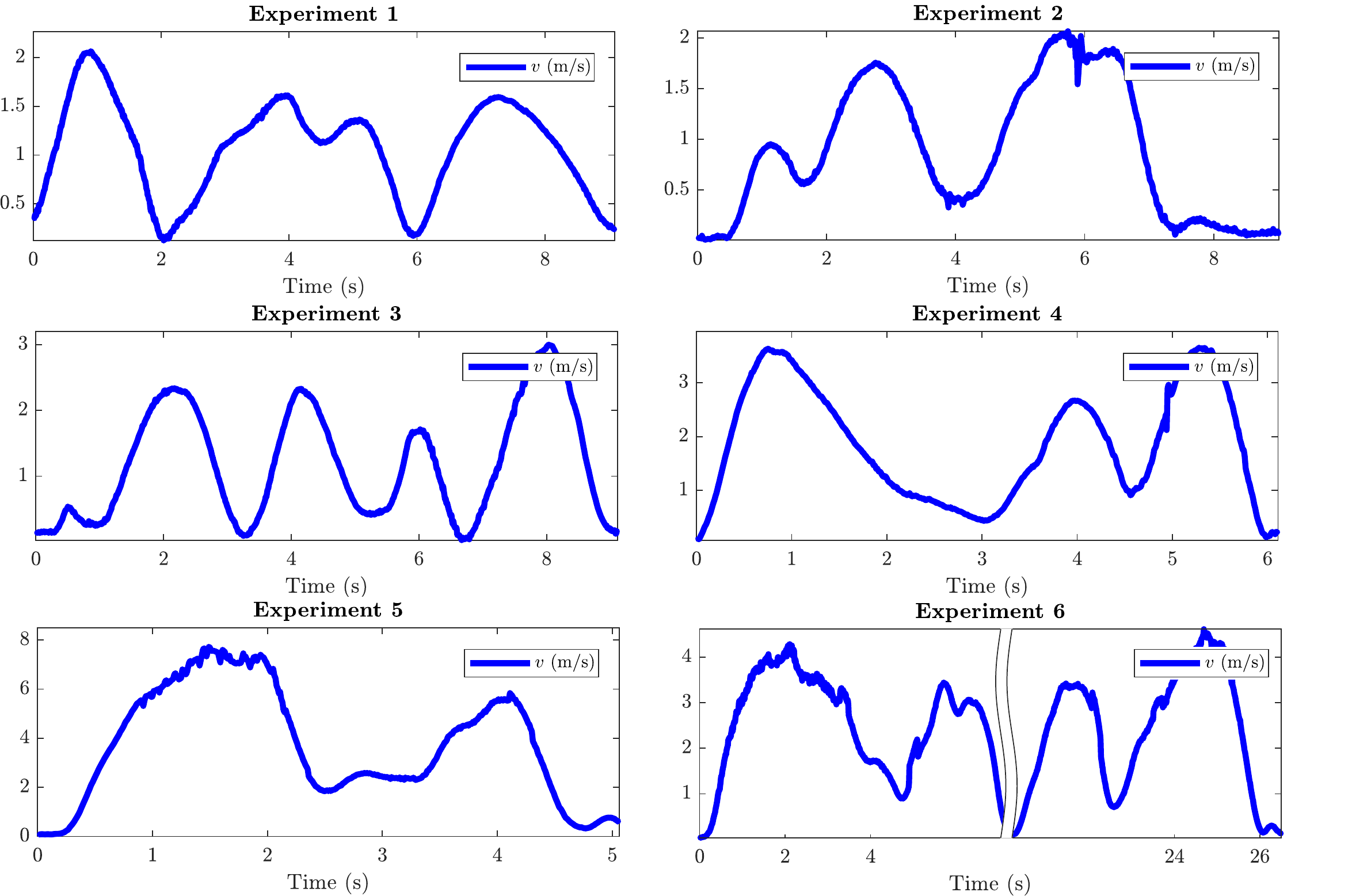}
	\caption[]{Velocity plots of all the UAV hardware experiments. \add{This velocity is the estimated velocity of the UAV, obtained by applying finite differences to the ground truth position measurements of an external motion capture system. This leads to some noisy estimates, especially for the high velocities of experiments 5 and 6. Moreover, these positions measurements are not available when the UAV is passing below an obstacle, which produces also noisy velocity estimates at those points. This happens in experiment 2 at $t=4.0$ s and $t=5.9$ s and in experiment 4 at $t=5.0$ s.}}
	\label{fig:velocity_profiles_hw}
\end{figure*} 

\begin{figure}[]
	\centering
	\includegraphics[width=1\columnwidth]{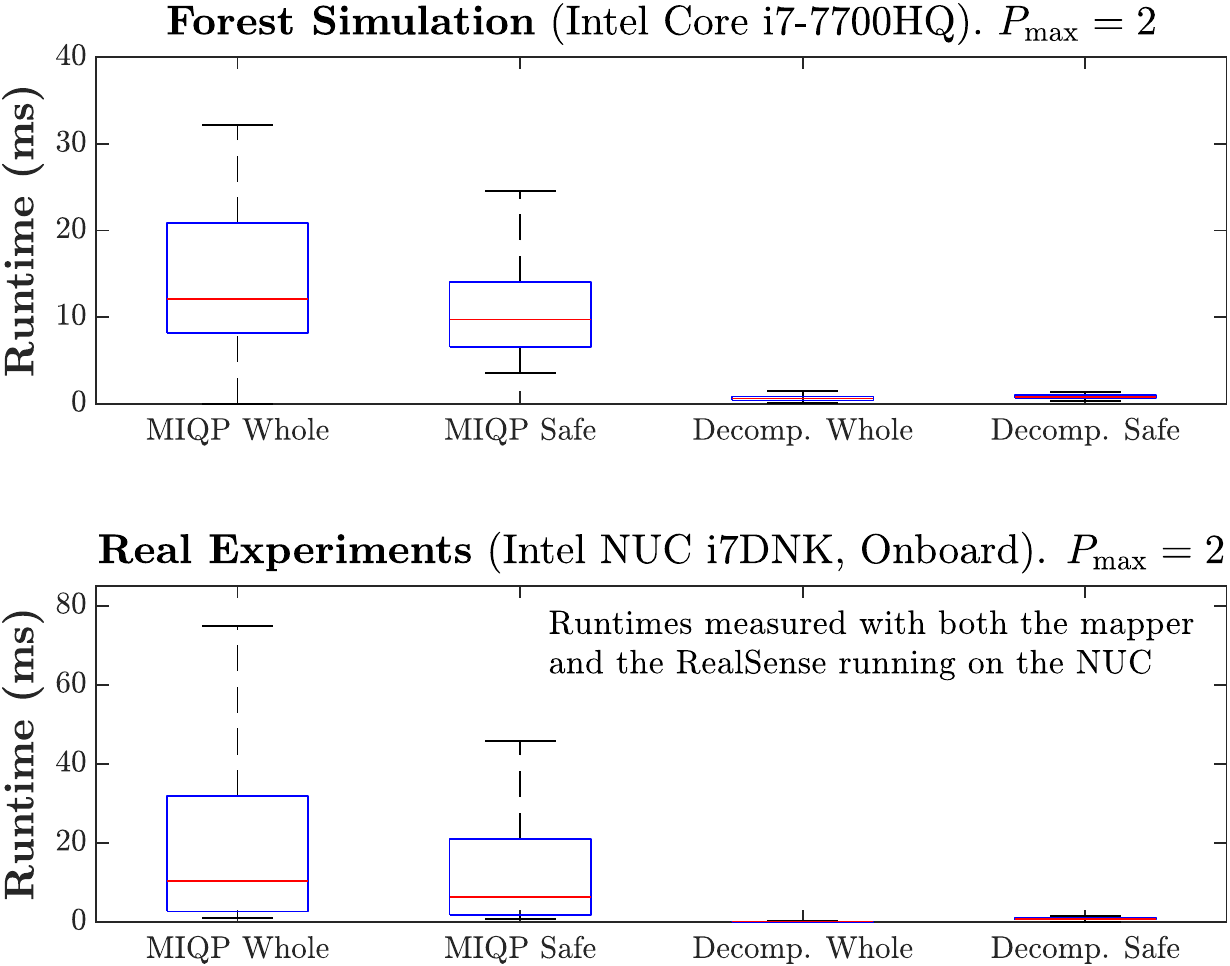}
	\caption[Timing breakdown for the forest simulation and for the hardware experiments]{ Timing breakdown for the forest simulation and for the real hardware experiments. The parameters used are $P_{\text{max}}=2$, $N=10$ for the Whole Trajectory, and $N=7$ for the Safe Trajectory.}
	\label{fig:timing_real_and_sim}
\end{figure} 

\add{\section{Hardware results}}\label{sec:HWresults}

The UAVs used in the hardware experiments are shown in Fig.~\ref{fig:drone}. A quadrotor was used in the experiments 1-4, and a hexarotor was used in the experiments 5 and 6. In both UAVs, the perception runs on the Intel\textsuperscript{\textregistered} RealSense, the mapper and planner run on the Intel\textsuperscript{\textregistered} NUC, and the control runs on the Qualcomm\textsuperscript{\textregistered} SnapDragon Flight. 

\begin{figure*}[]

	\centering
	\includegraphics[width=\textwidth]{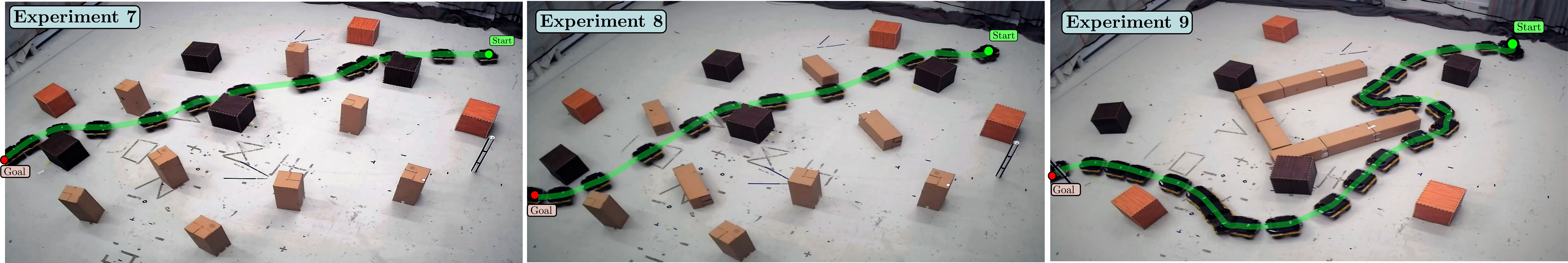}
	\caption{ Composite images of Experiments~7, 8 and 9. The ground robot must go from start \tikzcircle[black,fill=green]{2pt} to goal \tikzcircle[black,fill=red]{2pt}. Snapshots shown every 670~ms. To show the ability of FASTER to get out from bugtraps, only points in the depth image closer than 3 m were used to build the map in experiment 9.}
	\label{fig:exp789}
	
	\centering
	\includegraphics[width=\textwidth]{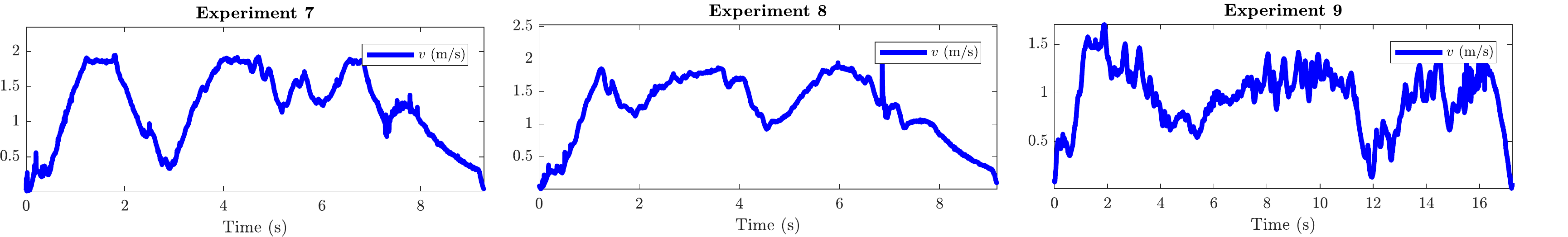}
	\caption{Velocity plots of the experiments 7, 8, and 9.}
	\label{fig:velocities_experiment789}
	
\end{figure*} 

The six hardware experiments done are shown in Figs.~\ref{fig:exp1}-- %
\ref{fig:exp6}. The corresponding velocity profiles are shown in Fig.~\ref{fig:velocity_profiles_hw}. The maximum speed achieved was $7.8$ m/s, in Experiment 5 (Fig.~\ref{fig:exp5}).
The first and second experiments (Fig.~\ref{fig:exp1} and \ref{fig:exp2}) were done in similar obstacle environments with the same starting point but with different goal locations. In the first experiment (Fig.~\ref{fig:exp1}), the UAV performs a 3-D agile maneuver to avoid the obstacles on the table. In the second experiment (Fig.~\ref{fig:exp2}) the UAV flies through the narrow gap of the cardboard boxes structure, and then flies below the triangle-shaped obstacle. In these two experiments, the maximum speed was $2.1$ m/s.

In the third and fourth experiments (Fig.~\ref{fig:exp3} and \ref{fig:exp4}), the UAV must fly through a space with poles of different heights, and finally below the cardboard boxes structure to reach the goal, achieving a maximum speed of $3.6$ m/s. 
Finally, in the fifth and sixth experiments (Fig.~\ref{fig:exp5} and \ref{fig:exp6}), the UAV is allowed to fly in a much bigger space, and has to avoid some poles and several cardboard boxes structures. In the fifth experiment (Fig.~\ref{fig:exp5}) the UAV achieved a top speed of $7.8$ m/s. In the sixth experiment (Fig.~\ref{fig:exp6}) the UAV was first commanded to go to a goal at the other side of the flight space, and then to come back to the starting position, achieving a top velocity of $4.6$ m/s.

\add{Fig.~\ref{fig:velocity_profiles_hw} shows the estimated velocity of the UAV, obtained by applying finite differences to the ground truth position measurements of an external motion capture system. This leads to some noisy estimates, in particular for the high velocities of experiments 5 and 6. Moreover, these positions measurements are not available when the UAV is passing below an obstacle, which produces also noisy velocity estimates at those points. This happens in experiment 2 at $t=4.0$ s and $t=5.9$ s and in experiment 4 at $t=5.0$ s.}

For $P_{\text{max}}=2$, the boxplots of the runtimes achieved on the forest simulation (measured on an Intel Core i7-7700HQ) and on the hardware experiments (measured on the onboard Intel NUC i7DNK with the mapper and the RealSense also running on it) are shown in Fig.~\ref{fig:timing_real_and_sim}. For the runtimes of the MIQP of the Whole and the Safe Trajectories, the 75\textsuperscript{th} percentile is always below $32$ ms.

\section{Extension to a ground robot} \label{sec:extension}

\add{We now show how, by generating 2-D trajectories instead of 3-D, and changing the controller, FASTER can also be extended for skid-steer robots. To track the trajectory obtained by MADER, we generate the linear and angular velocities using a PD controller based on the derivative of the tangential angle of the trajectory \cite{tangential2019} and the desired position and velocity. The commanded angular velocities of the wheels are then obtained from the desired angular velocities of the wheels using a PID. 
}

\begin{figure}[]
	\centering
	\includegraphics[width=.85\columnwidth,trim=60 190 60 120,clip]{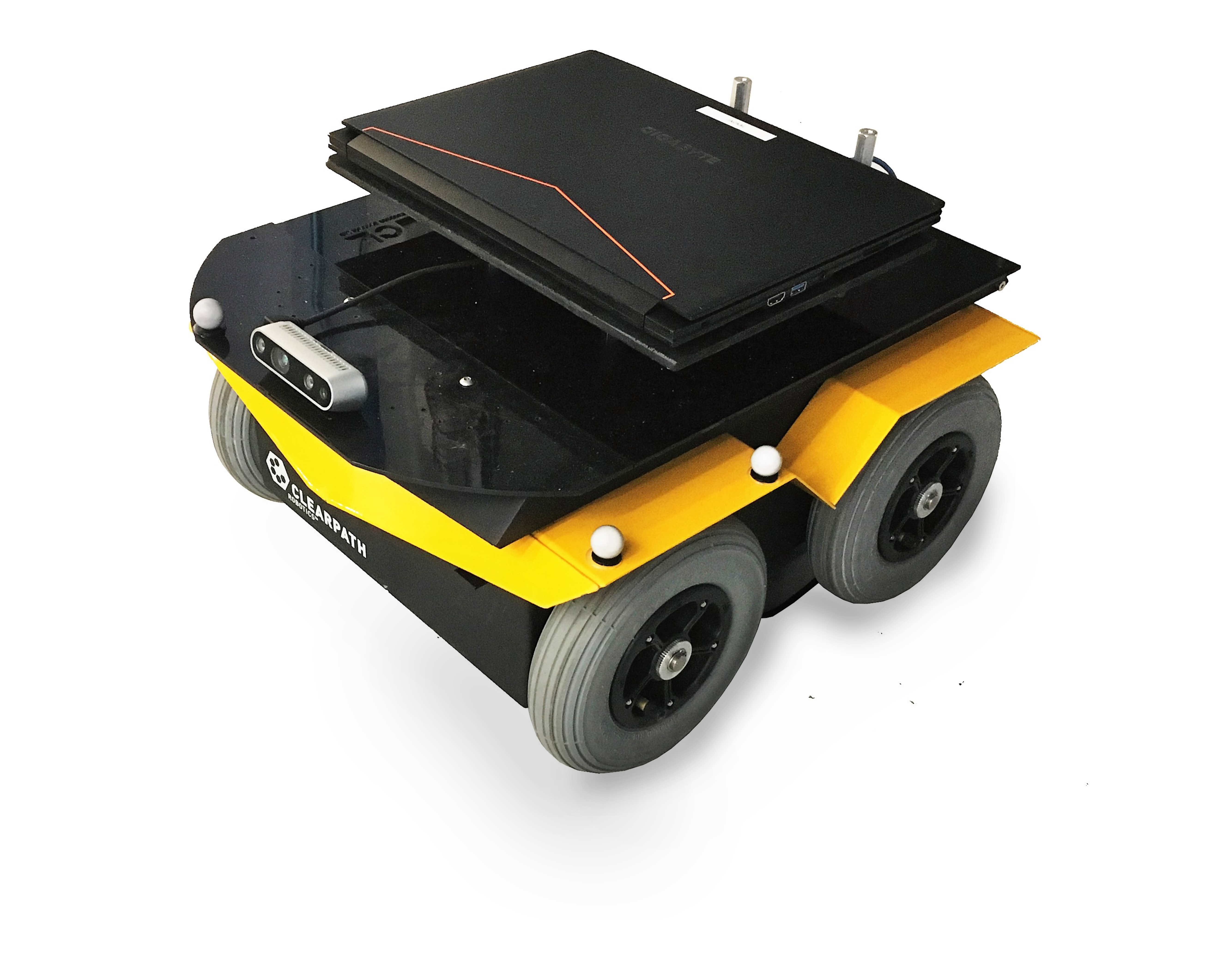}
	\caption{Ground robot used in the experiments. It is equipped with an Intel\textsuperscript{\textregistered} RealSense Depth Camera D435, and an i7-7700HQ laptop.  }
	\label{fig:jackal}
	\vspace{-0.2cm}
\end{figure}

Three different experiments were done with the ground robot (see Figs. \ref{fig:exp789}, \ref{fig:velocities_experiment789}, and \ref{fig:jackal}). An external motion capture system was used to estimate the position and orientation of the robot. Experiments 7 and 8 were done in obstacle environments similar to the random forest. The maximum speeds achieved for the experiments 7 and 8 were $1.95$ m/s and $2.22$ m/s respectively. Note that the maximum speed specified for this ground robot is $\approx 2$ m/s \cite{jackalspecif2019}.

To test the ability of FASTER to reuse the map built, the setup for experiment 9 was a bugtrap environment, and only points in the depth image closer than 3~m were used to build the map. The robot first enters the bugtrap because it does not see the end of it. Once the robot detects that there is no exit at the end of the bugtrap, it turns back, exits the bugtrap, passes through its left and avoids some new obstacles to finally reach the goal. The maximum speed achieved in this experiment was $1.70$ m/s

\section{CONCLUSIONS AND FUTURE WORK}
\label{sec:conclusions_future_work}
This work presented FASTER, a fast and safe planner for agile flights in unknown environments. The key properties of this planner is that it leads to a higher nominal speed than other works by planning both in $\mathcal{U}$ and $\mathcal{F}$ using a convex decomposition, and ensures safety by having always a Safe Trajectory planned in $\mathcal{F}$ at the beginning of every replanning step. FASTER was tested successfully both in simulated and in hardware flights, achieving velocities up to $7.8$ m/s. Finally, we showed how FASTER is also applicable to skid-steer robots, achieving hardware experiments at $2$ m/s. 

\add{Our algorithm has also some limitations: In environments where the planning horizon is not very large (as in all the experiments shown in this article), $2-4$ polyhedra usually suffice, and our algorithm maintains computational tractability. However, for large known worlds (for example if a map of the environment already exists beforehand), a long planning horizon may require more than 4 polyhedra, which, as shown in Fig.~\ref{fig:timing_all}, will increase the computation time. One possible way to address this is to solve the interval allocation only in the polyhedra that are close to the current position of the UAV, and force a predefined interval and time allocation for the polyhedra that are farther in the planning horizon. 
Moreover, we also noticed how important the choice of the point $R$ is: As discussed in Sec. \ref{subsec:complete_algorithm}, if the point $R$ is chosen very close to the unknown space, it may lead to infeasibility of the optimization problem associated with the Safe Trajectory. However, if $R$ is chosen very close to $A$, then the UAV may not have enough time to replan in the next iteration, which will lead to keep executing the previous trajectory, and may eventually decrease the nominal speed of the flight. Nonheuristic ways to solve this tradeoff seems like a promising direction for future work. Further future work includes the relaxation of  the assumption \ref{assumption_theorem}:} we plan to include the uncertainty associated with the map (due to estimation error and/or sensor noise) in the replanning function, and to extend this planner to dynamic environments. We also plan to use onboard estimation algorithms like VIO instead of an external motion capture system for the real hardware experiments. 

Finally, another promising future work is the reduction of the computation times of the time allocation approaches. Experiments in Sec.~\ref{subsec:timeVsIntervalAlloc} use a generic nonconvex solver to optimize the time allocation, which may be inefficient in some situations. Exploitation of the structure of the time allocation problem and/or the use of hierarchical optimization could help to reduce the associated computation times \cite{sun2020fast, tang2020enhancing}. This could potentially avoid the use of binary variables needed for the interval allocation, or allow the optimization of \emph{both} the interval and the time allocation in the trajectory planning problem.

\section*{Acknowledgment}
The authors would like to thank Pablo Tordesillas (ETSAM-UPM) for his help with some figures, to Parker Lusk and Aleix Paris (ACL-MIT) for their help with the hardware, and to Helen Oleynikova (ASL-ETH) for the data of the forest simulation. The authors would also like to thank John Carter and John Ware (CSAIL-MIT) for their help with the mapper used. This work was supported in part by  Defense Advanced Research Projects Agency (DARPA) as part of the Fast Lightweight Autonomy (FLA) program grant number HR0011-15-C-0110. Views expressed here are those of the authors, and do not reflect the official views or policies of the Department of Defense or the U.S. Government. The hardware was supported in part by Boeing  Research  and  Technology. The first author of this article was also financially supported by La Caixa fellowship.

\ifCLASSOPTIONcaptionsoff
  \newpage
\fi

\bibliographystyle{IEEEtran}
\bibliography{ref}

\vspace{4cm}
\begin{IEEEbiography}[{\includegraphics[width=1in,height=1.25in,clip,keepaspectratio]{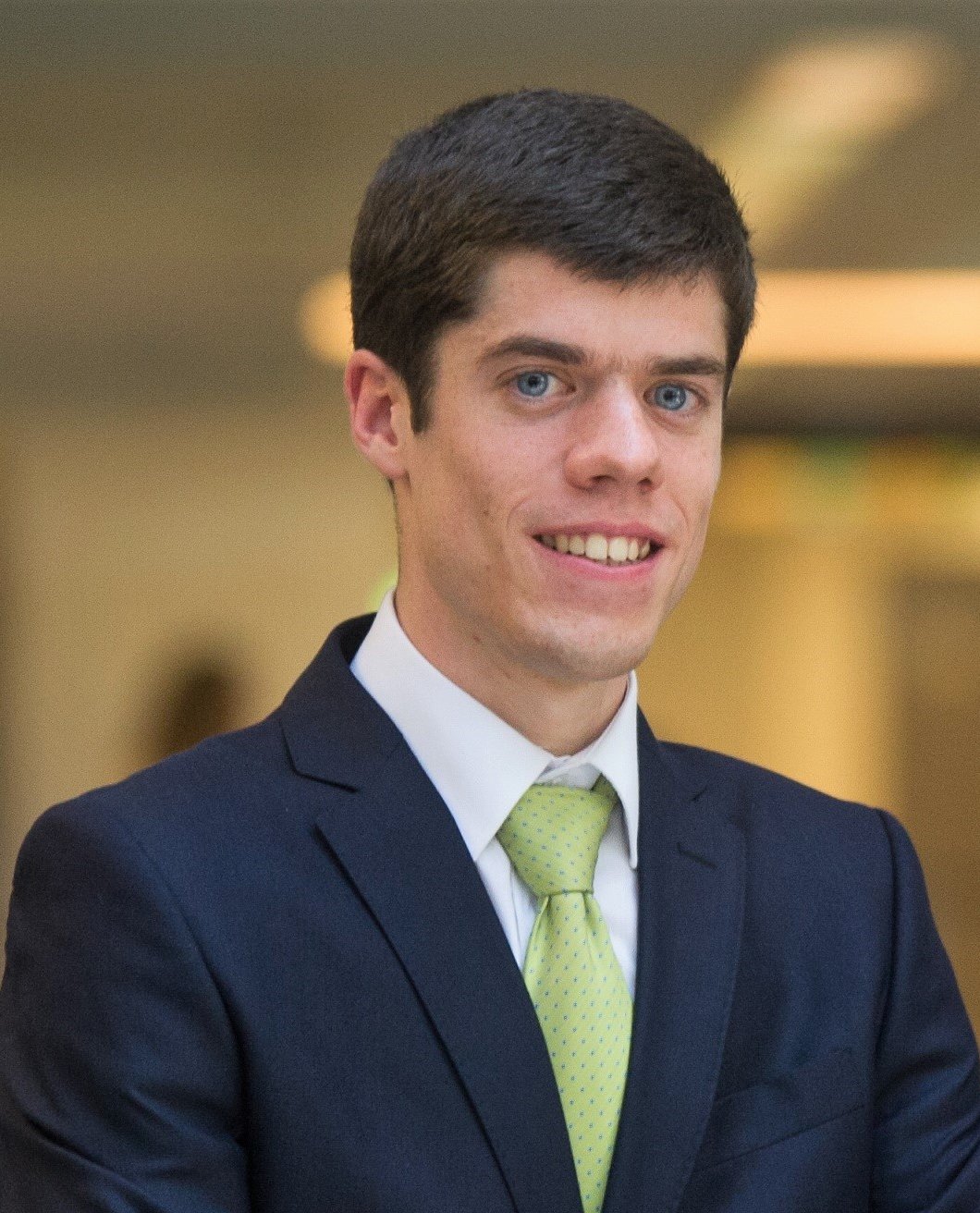}}]{Jesus Tordesillas}
(Student Member, IEEE)
received the B.S. and M.S. degrees in Electronic engineering and Robotics from the Technical University of
Madrid (Spain) in 2016 and 2018 respectively. He then received his M.S. in Aeronautics and Astronautics from MIT in 2019. He is currently pursuing the PhD degree with the Aeronautics and
Astronautics Department, as a member of the Aerospace Controls Laboratory (MIT) under the supervision of Jonathan P. How. 
His research interests include path planning for UAVs in unknown environments and optimization. He held an internship position at the NASA Jet Propulsion Laboratory, working with the Robotic Aerial Mobility Group. His work was a finalist for the Best Paper Award on Search and Rescue Robotics in IROS 2019.
\end{IEEEbiography}

\begin{IEEEbiography}[{\includegraphics[width=2.9in,height=1.25in,clip,keepaspectratio]{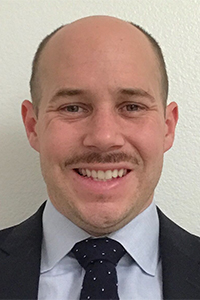}}]{Brett T. Lopez}
(Student Member, IEEE) is a Postdoctoral Scholar at the NASA Jet Propulsion Laboratory in the Robotic Aerial Mobility Group where he leads a team of engineers and researchers designing the next generation of autonomous aerial robots for the DARPA Subterranean Challenge. He obtained his PhD (2019) and SM (2016) from MIT working with Prof. Jonathan How. He obtained his BS (2014) from UCLA where he received the Aerospace Engineering Outstanding Bachelor of Science award. His research establishes performance guarantees for complex autonomous systems through nonlinear/adaptive control theory and optimization.
\end{IEEEbiography}

\begin{IEEEbiography}[{\includegraphics[width=2.9in,height=1.25in,clip,keepaspectratio]{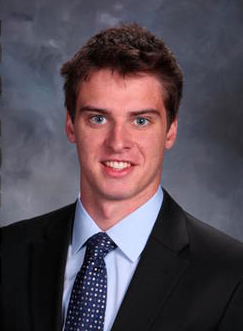}}]{Michael Everett}
(Student Member, IEEE) is a Ph.D. Candidate at the Aerospace Controls Laboratory at MIT. He received the SM degree (2017) and the SB degree (2015) from MIT in Mechanical Engineering. His research addresses fundamental gaps in the connection of machine learning and real mobile robotics. He was an author of works that won the Best Paper Award on Cognitive Robotics at IROS 2019, the Best Student Paper Award and finalist for the Best Paper Award on Cognitive Robotics at IROS 2017, and finalist for the Best MultiRobot Systems Paper Award at ICRA 2017. 
\end{IEEEbiography}

\begin{IEEEbiography}[{\includegraphics[width=2.9in,height=1.25in,clip,keepaspectratio]{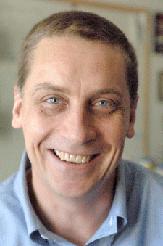}}]{Jonathan P. How}
(Fellow, IEEE) received the
B.A.Sc. degree from the University of Toronto (1987), and the S.M. and Ph.D. degrees in aeronautics and astronautics from MIT (1990 and 1993). Prior to joining MIT in 2000,
he was an Assistant Professor at Stanford University. He is currently the
Richard C. Maclaurin Professor of aeronautics and astronautics at MIT. Some of his awards include the IEEE CSS Distinguished Member Award (2020), AIAA Intelligent Systems Award (2020), 
IROS Best Paper Award on Cognitive Robotics (2019), and the AIAA Best
Paper in Conference Awards (2011, 2012, 2013). 
He was the Editor-in-chief of IEEE Control Systems Magazine (2015--2019), is a Fellow of AIAA, and 
was elected to the National Academy of Engineering in 2021.
\end{IEEEbiography}

\end{document}